\documentclass{article}
\usepackage{style-socsim}

\title{Beyond DAGs:  Modeling Causal Feedback with Fuzzy Cognitive Maps}
\author{Osonde Osoba\footnote{\textbf{Osonde A. Osoba, Ph.D.} (oosoba@rand.org) is a researcher at the RAND Corporation and a professor at the Pardee RAND Graduate School, Santa Monica, CA, USA. He received his Ph.D. in Electrical Engineering from the University of Southern California.}, \;
Bart Kosko\footnote{\textbf{Bart Kosko, Ph.D., J.D.} (kosko@usc.edu) is a professor of Electrical and Computer Engineering and Law at the University of Southern California in Los Angeles.}}
\date{}

\begin{document}
\maketitle

\begin{abstract}
Fuzzy cognitive maps (FCMs) model feedback causal relations in interwoven webs of causality and policy variables.
FCMs are fuzzy signed directed graphs that allow degrees of causal influence and event occurrence.
Such causal models can simulate a wide range of policy scenarios and decision processes.
Their directed loops or cycles directly model causal feedback.
Their nonlinear dynamics permit forward-chaining inference from input causes and policy options to output effects.
Users can add detailed dynamics and feedback links directly to the causal model or infer them with statistical learning laws.
Users can fuse or combine FCMs from multiple experts by weighting and adding the underlying fuzzy edge matrices and do so recursively if needed.
The combined FCM tends to better represent domain knowledge as the expert sample size increases if the expert sample approximates a random sample.
Many causal models use more restrictive directed acyclic graphs (DAGs) and Bayesian probabilities.
DAGs do not model causal feedback because they do not contain closed loops.
Combining DAGs also tends to produce cycles and thus tends not to produce a new DAG.
Combining DAGs tends to produce a FCM.
FCM causal influence is also transitive whereas probabilistic causal influence is not transitive in general.
Overall:  FCMs trade the numerical precision of probabilistic DAGs for pattern prediction, faster and scalable computation, ease of combination, and richer feedback representation.
We show how FCMs can apply to problems of public support for insurgency and terrorism and to US-China conflict relations in Graham Allison's Thucydides-trap framework.
The Thucydides-trap FCM predicts war-like conflict for nearly $80\%$ of input scenarios where we treat all such scenarios as equally likely\footnote{Sample code for FCM modeling and simulation models for examples discussed below are available at: \url{https://github.com/oosoba/fcm-evolve}}.
{The appendix gives the textual justification of the Thucydides-trap FCM and extends our earlier theorem~\parencite{osoba-kosko2017} that shows how a concept node transitively affects downstream nodes.}
The more general result shows the transitive and total causal influence that upstream nodes exert on downstream nodes.
\end{abstract}

\section{Fuzzy Cognitive Maps for Causal Modeling}

FCMs allow users to quickly draw causal diagrams of complex social or other processes \parencite{kosko1986}.
These causal diagrams can have closed loops or paths in them.
The closed loops directly model feedback among the causal concepts or nodes.
They are fuzzy because the causal arrows in the diagrams admit degrees or shades of gray.

Users can make what-if predictions with a given FCM:  What happens given this input policy?
The predictions are not numerical predictions.
They are pattern predictions or repeating sequences of events.
Users can also ask causal why questions:  Why did these events occur?
Different users can also fuse or combine their FCM diagrams into a unified FCM.

These FCM techniques are quite general and have led to numerous applications ~\parencite{papageorgiou2013review, glykas2010,  fcm-papageorgiou2013,  amirkhani2017review}.
A recent application developed an FCM representation of how `Brexit' scenarios could affect energy demand in the United Kingdom~\parencite{ziv-brexit2018}.
Another application used FCM techniques to represent the numerous social factors involved in homelessness ~\parencite{mago2013analyzing}.

A FCM uses an arrow or directed-graph \emph{edge} to describe how one concept \emph{node} causally affects another concept node.
So a FCM represents causality as a directed edge in graph.
The causality can be partial or fuzzy because the directed edge has a numerical magnitude that states the degree or intensity of causality.
A FCM model of military deterrence might include concept nodes for the degree of national military capability and for the degree of a threat's credibility.
The directed causal edges among two or more such concept nodes define a directed graph or \emph{digraph}.
A key aspect of FCMs is that in general their digraphs contain cycles or feedback loops.
This means that most FCMs are not directed \emph{acyclic} graphs or DAGs.

\subsection{Fuzz as a Matter of Degree}

Fuzz is a term of art in logic and decision science.
Fuzz denotes degree of truth or degree of causality.
Fuzzy logic extends binary logic to multi-valued logic and allows rule-based approximate reasoning \parencite{zadeh1965, Kosko1993}.
Decision theorist Richard Bellman was one of the founders of fuzzy theory and even defined abstraction as the estimation of a fuzzy-set membership curve \parencite{bellman1966abstraction}.

FCMs are fuzzy both in their causal edges and often in there concept nodes.
Fuzzy causal edges denote partial causality.
All-or-none causality can still occur but only as the endpoints of the spectrum of causal influence.
The same holds for the activation of a concept node.
Concept nodes are often binary in practice and in the examples below.
More sophisticated FCM models use fuzzy gray-scale concept nodes and may also use time lags.

Feedback loops in a FCM imply that a FCM is a dynamical system.
Inputs stimulate the FCM system.
The causal activation then swirls through the FCM until the FCM falls into a dynamical equilibrium.
The FCM stays in that dynamical equilibrium until a new input perturbs the system.
Most FCMs quickly reach an equilibrium.
The simplest equilibrium is a fixed-point attractor or a single state of the FCM that repeats over and over.
The more common equilibrium is a limit cycle where a sequence of states repeats over and over in a loop.
The equilibrium serves as the system's what-if prediction or forward inference from the input.
FCMs with binary nodes usually converge to a limit cycle.
The cycle of states is a form of pattern prediction.

\subsection{Comparison with Other Methods}

Two other approaches to modeling causal worlds are system-dynamics (SD) models~\parencite{sterman2000, abdelbari2018} and Bayesian belief networks (BBNs)~\parencite{pearl2009causality}.

System-dynamics models allow users to represent and simulate causal interactions within relatively complex systems.
But SD models have static parameters.
Domain experts or random experiments often choose the parameters of the subsystems and their interconnections.
FCMs admit both data-driven and expert-driven adaptation of the model structure \emph{and} the model parameters.
Statistically learning algorithms estimate causal edges from training data.
Experts can also state edge values directly.
SD models often account for stochasticity by using sensitivity analyses at the end of modeling.
FCMs build uncertainty into their very structure.

BBNs model uncertain causal worlds with conditional probabilities.
A user must first state a known joint probability distribution over all the nodes of the directed graph.
This may not be practical for large numbers of nodes.
Forward inference on a BBN also tends to be computationally intensive because it marginalizes out nodes.
The directed graph is usually \emph{acyclic} graph and thus has no closed loops.
The acyclic structure simplifies the probability structure but it ignores feedback among the causal units.
This is an important limitation for social and behavioral models that often feature causal loops.
The acyclic constraint also makes it hard to combine BBN causal models from multiple experts because such combination may well produce a cycle.
FCMs contain causal loops by design.
Combining FCMs only tends to increase the loop or feedback structure and thereby produce richer feedback dynamics.
It also tends to improve the modeling accuracy in many cases.
FCM forward-inference process is simple and fast because it uses only vector-matrix multiplication and thresholding.
But BBNs do give \emph{precise} numerical probability descriptions and not the mere pattern predictions of FCMs.
So \emph{FCMs trade numerical precision for pattern prediction, faster and scalable computation, ease of combination, and richer feedback representation.}

FCMs offer many advantages for social-scientific modeling and simulation.
FCMs can capture the causal beliefs of domain experts.
FCM dynamics can reveal global ``hidden patterns'' in small or large causal webs.
These patterns tend to far from obvious when examining a large-scale FCM model.
The domain expert expert expresses his beliefs about how relevant factors causally relate to one another.
Some examples of such expert-based FCMs include public support for terrorism~\parencite{osoba-kosko2017}, blood-clotting reactions~\parencite{taber-yager-helgason2007}, and medical diagnostics~\parencite{stylios2008fuzzy}.

Analysts can also convert documents into FCMs.
We show below an example translating Graham Allison's textual descriptions of his proposed Thucydides Trap of international relations into a predictive FCM.
Combining such documents and their FCMs gives one way to approach ``big knowledge'' because combing FCMs always results in a new FCM.
FCMs can also apply driving agents' behaviors or decisions in agent-based models.
Earlier FCMs ~\parencite{dickerson-kosko1993} show how lone or combined FCMs can govern the behavior of virtual actors.
These behaviors can be simple or complex depending on the driving FCM.
Applying learning laws to the agents' FCMs can simulate the effect of agents learning new behaviors.

\section{Fuzzy Cognitive Maps as Models of Causal Feedback}

FCMs offer a practical way to model and process the interwoven causal structure of policy and decision problems.
Numerous FCM applications have appeared in recent years.
They range from control engineering and medicine to policy analysis and social modeling \parencite{papageorgiou2013review, glykas2010,  fcm-papageorgiou2013,  amirkhani2017review}.

FCMs are fuzzy causal signed directed graphs \parencite{kosko1986}.
The digraphs are fuzzy because in general both their directed causal edges and their concept nodes admit degrees.
So they can assume more values than just the bivalent extremes of on or off.
Their multivalued  causal edges directly model degrees of causal influence.
Users can express their causal and policy models by drawing signed and weighted causal edges between concept nodes.
This does require specifying the nonlinear dynamical structure of the concept nodes.
Statistical machine-learning algorithms can also infer or tune the edges using training data if representative data is available.
A differential Hebbian learning law can approximate a FCM's directed edges of partial causality using time-series training data.

Figure \ref{fig:herd-dyn-ex} shows a FCM fragment that models a simple undersea causal web of dolphins in the presence of sharks or other survival threats.
This simple model uses concepts that are either active or not active.
The directed causal edges likewise show all-or-none causal increase or decrease (missing edges have zero weight).
Section 3 shows how to make what-if inferences or predictions with this simple FCM that has binary concept nodes and trivalent causal edges.
The inference process uses only vector-matrix multiplication and thresholding.
More complex FCMs can activate concept nodes with soft thresholds or bell curves or other nonlinear transformations of inputs to outputs.
Figure 3 in \parencite{osoba-kosko2017} shows six examples of soft thresholds.
Figure \ref{fig:fcm-combo} shows how to combine FCMs even when some of the FCMs use different concept nodes.
A causal learning law can approximate the causal edge values given time-series data of the concept nodes.
Figure \ref{fig:fcm-adapt} shows the approximation path of a causal edge from a FCM that models public support for insurgency and terrorism.

A FCM's overall cyclic signed digraph structure resembles a feedback neural or semantic network.
FCMs also resemble causal loop diagrams and systems-dynamics models \parencite{sterman2000, abdelbari2018} because nonlinear difference or differential equations describe FCM concept nodes (and sometimes FCM causal edges).
A FCM's directed graph structure permits inference through forward chaining.
It also allows the user to control the level of causal or conceptual granularity.
A FCM concept node can itself be part of another FCM or of some other nonlinear system.
Feedforward fuzzy rule-based systems can also model the input-output structure of a concept node just as they can model a single causal edge that connects one concept node to another.
Such fuzzy systems are uniform function approximators if they use enough if-then rules \parencite{kosko-fat1994}.
A uniform approximation allows the user to pick the approximation error-tolerance level in advance for all inputs.
Their rule bases adapt using both unsupervised and supervised learning laws \parencite{kosko-fuzeng, osoba-mitaim-kosko-SMC2011}.

A FCM tends to have many cycles or closed loops in its fuzzy directed graph.
These cycles directly model causal feedback from self-loops to multi-path causality.
The cycles also produce complex nonlinear dynamics.
FCM causal inference maps input states or policies to equilibria of the nonlinear dynamical system.
Users can also step through time slices of the FCM dynamical system to at least partially unfold the system in time.

A FCM's feedback structure contrasts with the acyclic structure of Bayesian belief networks (BBNs).
BBNs form the basis of Pearl's popular model of causal inference~\parencite{pearl2009causality}.
Rubin's counterfactual approach to causality is a related statistical model \parencite{rubin2005, imbens2015causal}.
BBNs for causal inference assign probabilities to a directed \emph{acyclic} graph (DAG).
Their acyclic causal tree structure rules out feedback pathways.
This strong acyclic assumption greatly simplifies the probability calculus on such digraphs and may permit finer control when propagating probabilistic beliefs.
It allows the sum-product algorithm to compute marginal node probabilities from a known joint probability distribution on all the nodes.
FCMs do not produce such probability estimates.
But the acyclic structure is hard to reconcile with the inherent and extensive feedback causality of large-scale social phenomena from social networks to state-versus-state wars.
These social systems are high-dimensional nonlinear dynamical systems.
They have dynamics because they have feedback loops.

The loop-free structure of DAGs also makes it hard to combine causal models from multiple experts.
Combining DAGs need not produce a new DAG.
So combining DAGs is not a closed graph-theoretical operation in general.
Some experts will tend to draw opposing causal arrows between nodes.
Others will tend to add links that create multi-node closed loops.
This helps explain why BBNs and tree-based expert systems often have relatively few nodes.
A compounding factor is the sheer computational complexity of belief propagation.

FCMs naturally combine into a new FCM.
So FCM combination is a closed graph-theoretical operation.
The user can combine any number of FCMs by adding their underlying augmented adjacency matrices.
This gives a simple and powerful way to perform expert knowledge fusion.
Such knowledge fusion is a key function in many defense and intelligence decision-making processes~{\parencite{davis-rand2015, davis-wsc2015, kosko1986combination, kosko-hidden1988,  taber1991,taber-yager-helgason2007}}.
Figure 3 shows how two weighted medical FCMs combine into a fused FCM.
It shows the minimal fusion case of combining two FCMs with overlapping concept nodes.
The fused FCM is a weighted average or probability mixture of the combined FCMs.
Averaging the binary augmented adjacency matrices of ordinary DAGs always produces a FCM.

The strong law of large numbers shows that in many cases this fused FCM converges with probability one to the population FCM of the sampled FCMs  \parencite{taber-yager-helgason2007}.
This result holds formally if the FCM edge values from the combined experts approximate a statistical random sample with finite variance.
A random sample is sufficient for this convergence result but not necessary.
A combined FCM may still give a representative knowledge base when the expert responses are somewhat correlated or when the experts do not all have the same level of expertise.
Users can also construct FCMs from written sources such as policy articles or books or legal testimony.
They can also use statistical learning algorithms to grow FCMs from sample data.
FCMs can in this way address the growing representational problems of big data and what we have called ``big knowledge'' \parencite{fcm-papageorgiou2013}.

FCM concept nodes can also be part of some other FCM or of some other nonlinear system.
Feedforward fuzzy rule-based systems can also model the input-output structure of a concept node just as they can model a single causal edge that connects one concept node to another.
These fuzzy systems are uniform function approximators if they use enough if-then rules \parencite{kosko-fat1994}.
We can also generalize the causal edge values to equal the outputs of such fuzzy rule-based systems or generalized probability mixtures \parencite{kosko2017generalized,kosko2018additive}.
We can further put a probabilistic structure on top of the fuzzy causal structure.
This chapter looks only at the fuzzy causal structure.

Inference on a graph maps nodes to nodes.
Forward inference maps observable \emph{evidence} nodes or variables $C_{Ev}$ to \emph{output} nodes or variables $C_O$.
Backward inference maps output nodes back to input nodes.
Probabilistic inference on a graph computes the conditional probability $P(C_O|C_{Ev})$ of output nodes $C_O$ given a state vector on observable evidence nodes $C_{Ev}$.
This computation involves a complex marginalization operation on general DAGs.
It also assumes that the user knows the closed-form joint probability distribution on all the nodes.
The computation often requires complex message-passing algorithms such as belief propagation {\parencite{yedidia2001, murphy2002}} or the more general junction-tree algorithm {\parencite{wainwright-jordan2008}}).
This probabilistic inference or computation is NP-hard in general~{\parencite{dagum-luby1993,russell-norvig2016}}.
And the problem of learning the underlying Bayesian network from data is an NP-complete problem in general~{\parencite{chickering1996,chickering2004}}.

Loops or cycles in a causal graph model may render exact probabilistic inference difficult if it is even feasible.
Inexact or \emph{loopy} inference schemes can give useful approximations in many cases.
These methods include loopy belief propagation, variational Bayesian methods, mean-field methods, and some forms of Markov-chain Monte Carlo simulation~{\parencite{pearl2014,wainwright-jordan2008,beal-ghahramani2003, beal-ghahramani2006}}.
But loopy algorithms need not converge.
Nor does their use overcome the NP-hard complexity of probabilistic belief propagation.
It may instead only compound the computational complexity.

FCM forward inference uses light computation.
It requires only vector-matrix multiplication and nonlinear transformations of vectors.
The transformations are often hard or soft thresholds.
So FCM forward inference has only polynomial-time complexity.
This means that FCMs scale fairly well to problems with high dimension or multiple concept nodes.
Simple binary-state FCMs converge quickly to limit-cycle equilibria given an input stimulus {\parencite{kosko-nnfs}}.
FCMs with more complex node nonlinearities can converge to aperiodic or chaotic equilibria if sufficiently complex nonlinearities describe the concept nodes.

FCMs do have at least two structural limitations.
The first is that a user may not be able to use some predicted outcomes.
A user may find it hard to interpret a FCM's what-if predictions because they are equilibria of a highly nonlinear dynamical system.
Simple predictions may be limit cycles that consist of only a few ordered system states.
The dolphin example below is one such case.
One equilibrium output consists of four binary vectors that repeat in sequence.
Other bit-vector limit cycles can consist of much longer ordered sequences.
Richer dynamics can produce aperiodic or chaotic equilibria in regions of the FCM state space.
Such predicted outcomes may have no clear policy interpretation if we cannot clearly associate the equilibrium attractor's region of the FCM state space with a temporally ordered sequence of FCM states.

A more fundamental limitation is that FCMs do not easily answer why questions given observed outcomes.
FCMs do not easily admit backward inference from effects to causes because FCM nodes are neural-like nonlinear mappings of causal inputs to outputs.
A user may need to test a wide range of random inputs to see which FCM states map to or near a given observed or conjectured output state.
We often check all possible input policy states to find all output equilibria.
We did this with the Thucydides-trap FCM below by ``clamping on'' input policy variables while the FCM dynamical system converged.
We also tested this FCM's pattern predictions against those of a thresholded FCM whose edge values were the trivalent extremes of $-1, 0, \text{or } 1$.
The thresholded FCM gave similar equilibrium predictions.

The next sections explore FCMs in some detail.
We conclude with two FCM policy applications.
The first shows how FCMs can assist in modeling the many causal and policy factors involved in public support for insurgency and terrorism.
The second shows how a FCM model can give insight into Allison's recent ``Thucycides trap'' model of US-China conflict.

\section{Probabilistic vs. Fuzzy Models of Causality}

Causality is uncertain in general.
Different uncertainty model tend to produce different causal models.
Probability and fuzz or fuzziness are two such different types of uncertainty.
Probability or randomness describes whether an event occurs.
Fuzz or vagueness describes the degree to which an event occurs.
These two types of uncertainty produce different causal models even though both describe uncertainty with numbers in the unit interval $[0, 1]$.

Randomness and fuzz can also combine in their descriptions.
This takes some care even at the level of natural language.
The statement that ``There is a 20\% chance of light rain tomorrow'' asserts the random occurrence of the fuzzy event of light rain.
All rain is both light rain and not light rain to some degree.
It is a matter of degree absent a binary definition that literally specifies which rain drop is light rain and which is not.
Zadeh first showed in 1968 how to combine these two distinct types of uncertainty and in this order.
The probability of a fuzzy event is just the expected value of a measurable multivalued indicator function\parencite{zadeh1968}.

A different juxtaposition of uncertainty types holds for the statement that ``The probability of rain tomorrow is low.''
It asserts a fuzzy probability rather than the probability of a fuzzy event.
Its meaning also requires a fuzzy-set definition of low versus non-low probability.
Such combinations can occur in digraph models of causality.
This chapter focuses on the simpler cases of purely probabilistic causality and purely fuzzy causality.
It focuses on how these uncertainty models view the directed causal edge $e_{ij}$ from concept node $C_i$ to concept node $C_j$.
The FCM view of causality is simply that causality is what a fuzzy signed directed edge measures.
Active causality is the flow of concept-node activation or influence from node to node through an intervening directed causal edge.

A probabilistic view might cast the directed causal edge value $e_{ij}$ as the conditional probability $P(C_j | C_i)$ that $C_j$ occurs given that $C_i$ occurs.
An immediate problem is that $e_{ij}$ takes on negative values in the bipolar interval $[-1, 1]$ to indicate causal decrease.
There is a simple but somewhat costly way to address this.
The original FCM paper \parencite{kosko1986} showed how to introduce $n$ companion \emph{dis-concepts} to keep all causal edge values nonnegative and thus how to convert causal decrease into causal increase:  ``Extreme terrorism decreases government stability'' holds just in case ``Extreme terrorism increases government instability'' holds.
So dis-concepts negate the noun and not the adjective that modifies it.
Using dis-concepts doubles the number of concept nodes and expands the edge matrix $E$ to a $2n \times 2n$ matrix.
The technique does preserve more causal structure when combining multiple FCMs because then two combined edges of opposite polarity do not cancel each other out if they have the same magnitude.

There are two structural problems with viewing the directed (positive) edge $e_{ij}$ as the conditional probability $P(C_j | C_i)$.
The first problem is that conditional probability is not transitive.
This is telling since both logical and causal implication are transitive (at least on most reckonings):  $A$ causes $C$ if $A$ causes $B$ and if $B$ causes $C$.
But the transitive equality $P(C|A) = P(B|A) P(C|B)$ does not hold in general.
A simple counter-example takes any two disjoint or mutually exclusive events $A$ and $B$ with positive joint probabilities $P(A \cap C) > 0$ and $P(B \cap C) > 0$ if all three set events have positive probability.
Then $P(C|A) > 0$ but $P(B|A) P(C|B) = 0$ because $P(B|A) = \frac{P(A \cap B)}{P(A)} = 0$ since $A \cap B = \emptyset$.
An even starker counter-example results if $A \subset C$ because then $P(C|A) = 1$ while $P(B|A) P(C|B) = 0$.

The second and deeper problem with a probability interpretation of $e_{ij}$ is that it collides with the Lewis Impossibility Theorem \parencite{lewis1976, lewis1986}.
This triviality result and its progeny show that we cannot in general equate the probability $P(A \rightarrow B)$ of the logical if-then conditional $A \rightarrow B$ with the conditional probability $P(B|A)$.
The equality $P(A \rightarrow B) = P(B|A)$ holds only in the trivial case when $A$ and $B$ are independent and thus when there is no conditional relationship at all.
So a probabilistic transitive equality of the form $P(A \rightarrow C) = P(A \rightarrow B) P(B \rightarrow C)$ lacks a formal foundation in general.
One approach is to replace conditional probability with a more general probable equivalence relation.
This gives upper and lower conditional probabilities based on the general inequality that $P(A) P(B) \ge P(A \cap B) P(A \cup B)$ \parencite{kosko2004} because then $P(B|A) \le Q(B|A)$ if $Q(B|A) = \frac{P(B)}{P(A U B)}$.
The resulting conditioning interval does not directly address the basic prohibition that lies behind Lewis's triviality theorem.
So a meta-level heuristic may be the best we can make of probabilistic interpretations of the directed edge $e_{ij}$.

We stress that there are many ways to combine fuzz and probability in a FCM.
A system-level example is the averaging technique we discuss below on FCM knowledge combination.
The averaging computes a probability mixture of augmented signed fuzzy causal-edge matrices.
The law of large numbers can also apply to such averages in many cases.

Probability and fuzz can also apply at the level of the lone directed edge $e_{ij}$.

We can view the causal influence $e_{ij}$ from concept node $C_i$ to concept node $C_j$ as a \emph{mapping} $f$ from $C_i$ values to $C_j$.
The edge $e_{ij}$ maps more generally from all $n$ nodes to $C_j$.
Then a fuzzy system $F: X \rightarrow Y$ of $m$ fuzzy if-then rules can approximate any such mapping $f$.
The rules have the form ``If $X = A_j$ then $Y = B_j$'' where $A_j$ is a fuzzy subset of the input space and $B_j$ is a fuzzy subset of the output space.
The fuzzy system $F$ can suffer rule explosion when its input space has dimension.
A Watkins representation can exactly represent $f$ with just two rules if the real function $f$ is bounded and not constant.
Fuzzy rules can also absorb or approximate the prior, likelihood, and posterior probability density functions of modern Bayesian inference \parencite{osoba-mitaim-kosko-SMC2011}.
We note that fuzzy systems admit a complete probabilistic description in terms of mixtures of probability curves \parencite{kosko2017generalized, kosko2018additive}.
The rules roughly correspond to the mixed probability curves.
So a rule base of $m$ if-then fuzzy rules gives rise to a generalized probability mixture $p(y|x) = p_1(x) p_{B_1}(y|x) + \cdots p_m(x) p_{B_m}(y|x)$ where the generalized mixing weights $p_j(x)$ are nonnegative and sum to unity for each input $x$ and where the $m$ likelihood functions $p_{B_j}(y|x)$ also depend on $x$ in general.
The generalized mixture gives ia generalized version of the theorem on total probability.
So it gives at once a Bayes theorem that describes which rules fired to which degree for any given input and output.
The simplest case results when the fuzzy system $F$ has just one rule.
Then $F$ reduces to the constant function $f(x) = e_{ij}$.

We next show how to define a FCM and use it to make causal inferences.

\section{Forward Inference with Fuzzy Cognitive Maps} \label{sec:fcm-inf}

FCMs are fuzzy signed directed graphs that describe degrees of causality and webs of causal feedback~\parencite{kosko-nnfs, fcm-papageorgiou2013}.
Most FCMs have cycles or closed loops that model causal feedback.
FCMs can be acyclic and thus define directed trees.
This is rare in practice and implies that such a FCM has no feedback dynamics.

FCMs are fuzzy because their nodes and edges can be multivalent and so need not be binary or bivalent.
We now develop this fuzzy structure and apply it to FCMs.
A property or concept is fuzzy if it admits degrees and is not just black and white \parencite{zadeh1965, zadeh1973} 
Then the property or concept has borders that are gray and not sharp or binary.
This formally means that a subset $A$ of a space $X$ is properly fuzzy if and only if at least one element $x \in X$ belongs to $A$ to a degree other than $0$ or $1$.
Then the set $A$ breaks the so-called ``law'' of contradiction because then $A \cap A^c \neq \emptyset$ holds where $A^c$ is the complement set of $A$.
The set $A$ equivalently breaks the dual ``law'' of excluded middle because then $A \cup A^c \neq X$ holds.
Equality holds in these two ``laws'' just in case $A$ is an ordinary bivalent set.

A FCM concept node is fuzzy in general because it can take values in the unit interval $[0, 1]$.
So its values over time define a fuzzy set.
This implies that a concept node that describes a survival threat or any other property or policy both occurs and does not occur to some degree at the same time.
It cannot both occur $100\%$ and not occur $100\%$ at the same time.
The two percentages must sum to $100\%$.
Nor again does this preclude applying a probability measure to a concept node or fuzzy set.
The probability of a fuzzy event combines the two distinct uncertainty types of randomness and vagueness or fuzz (and formally involves taking the expectation of a measurable fuzzy indicator function \parencite{zadeh1968}).
So it makes sense to speak of the probability of a partial survival threat.
This differs from the compound uncertainty of a fuzzy probability such as the statement that the survival-threat probability is low or very high.
This chapter works only with fuzzy concept values

A directed causal edge $e_{ij}$ is also fuzzy because in general it takes on a continuum of values.
The edge can also have a positive or negative sign.
So it takes values in the bipolar interval $[-1, 1]$.
The use of ``dis-concepts'' can convert all negative causal edges into positive edge values \parencite{kosko1986}.

A simple FCM consists of $n$ concept nodes $C_j$ and $n^2$ directed fuzzy causal edges $e_{ij}$.
The $n$ concept nodes $C_1, C_2, \ldots , C_n$ are nonlinear and represent variable concepts or factors in a causal system.
They are nonlinear in how they convert their inputs to outputs.
The concept nodes can define concepts or social patterns that increase or decrease such as political instability or jihadi radicalism.
Or they can describe policies or control variables that increase or decrease such as weapons spending or foreign investment in a country.
The very first FCM published \parencite{kosko1986} dealt with concepts related to Middle East stability such as Islamic fundamentalism and Soviet imperialism and the strength of the Lebanese government.
The author based this first FCM on a 1982 newspaper editorial from political analyst Henry Kissinger titled ``Starting Out in the Direction of Middle East Peace.''

A concept node's occurrence or activation value $C_i(t_k)$ measures the degree to which the concept $C_i$ occurs in the causal web at time $t_k$.
It can also reflect the degree to which it is true that the $i$th node fires or appears in a given snapshot of the causal web at time $t_k$.
The FCM state vector $\mathbf{C}(t_k)$ gives a snapshot of the FCM system at time $t_k$.
The FCM edges themselves may be the weighted average of several experts as we explain below.

A FCM model must specify the nonlinear dynamics of the $n$ concept nodes \\
$C_1, C_2, \ldots , C_n$.
It must also specify the $n^2$ directed and signed causal edge values $e_{ij}$ that connect the $i$th concept node $C_i$ to $C_j$.
The edges can be time-varying functions in more general FCMs.

We start with the nonlinear structure of the concept nodes.
The $j$th concept node $C_j$ depends at time $t_k$ on a scalar input $x_j(t_k)$ that weights and aggregates all the in-flowing causal activation to $C_j$.
Then some nonlinear function $\Phi_j$ converts $x_j(t_k)$ into the concept node's new state $C_j(t_{k+1})$ at the next discrete time $t_{k+1}$.
The FCM literature explores discrete and continuous node models with a wide variety of nonlinearities and time lags \parencite{glykas2010,  fcm-papageorgiou2013}.
We present here the simplest case of a discrete FCM where each node's current state depends on an edge-weighted inner product of the node activity:
\begin{align}
	C_j(t_{k+1})  =  \Phi_j \left( \sum_{i=1}^n C_i(t_k) \; e_{ij}(t_k)  \;  +  \;  I_j(t_k) \right)
	\label{eq:node-nonlinearity}
\end{align}
where $I_j(t_k)$ is some external or exogenous forcing value or input at time $t_k$.
The simplest nonlinear function $\Phi_j$ is a hard threshold that produces bivalent or on-off concept node values:
 \begin{equation} \label{eq:threshold-node}
   	C_j(t_{k+1})     =     \begin{cases}
      0\ & \text{\it if  \;    $ \sum_{i=1}^n C_i(t_k) \; e_{ij}(t_k)  \;  +  \;  I_j(t_k)                 \leq  0$}\\
      1\ & \text{\it if  \;   $  \sum_{i=1}^n C_i(t_k) \; e_{ij}(t_k)  \;  +  \;  I_j(t_k)      >$   0}\
   \end{cases}
\end{equation}
for a zero threshold value.

Fixing the input $I_j$ as some very large positive (or negative) value ensures that $C_j$ stays on (or off) during an inference cycle.
We call this ``clamping'' on (or off) the $j$th concept node $C_j$.
We clamp one or more concept nodes to test a given policy or forcing scenario.
Clamping is the only way to drive policy or other nodes that have no causal fan-in from other concept nodes.
We show below how to model the sustained presence of a shark in the dolphin FCM of Figure \ref{fig:herd-dyn-ex}  by clamping on the fourth concept node.

Continuous-valued concept nodes often use a monotone increasing $\Phi_j$ nonlinearity such as the logistic sigmoid function.
But $\Phi_j$ can also be nonmonotonic.
This happens if it is a Gaussian or Cauchy probability density function.
It can also be multimodal by forming a mixture of such unimodal probability curves.
Then the Expectation-Maximization algorithm can tune the mixture parameters based on numerical training data \parencite{osoba2016noisy}.
Almost all concept nodes are monotonically nondecreasing in the FCM literature.
The causal-influence theorem below holds for such activation functions $\Phi_j$.

The logistic causal activation gives a soft threshold that approximates the hard threshold in (\ref{eq:threshold-node}) if the shape parameter $c > 0$ is large enough:
\begin{align}
	C_j(t_{k+1})   =  \frac{1}{1+\exp\Big{(}- c \sum_{i=1}^n C_i(t_k) \; e_{ij}(t_k)  \;  -  \;c  I_j(t_k)    \Big{)}}  \;. \label{eq:logistic-node}
\end{align}
The first graph in Figure 3 of \parencite{osoba-kosko2017} shows a logistic function and its sigmoidal or soft-threshold shape. 
Logistic units are popular in causal and neural-network learning algorithms because they smoothly approximate the on-off behavior of threshold units and still have a simple partial derivative of the form
\begin{align}
	\frac{\partial C_j(x)}{\partial x} = c \; C_j (1  -  C_j) > 0 \label{eq:logistic-factor}
\end{align}
if
\begin{align}
	C_j(x)   =  \frac{1}{1+\exp(- c x)}  \label{eq:logistic-derivative}
\end{align}
for scaling constant $c > 0$.
The positive derivative in (\ref{eq:logistic-factor}) greatly simplifies many learning algorithms.
We will also use it below to show the transitive product effect of the edges $e_{ij}$ in a causal inference.

We turn next to the causal edge values $e_{ij}$.
These values are constants during most FCM inferences.
Section 6 below shows how a version of the differential Hebbian learning law can learn and tune these causal edge values from time-series data.

The causal edge value $e_{ij}(t_k)$ in (\ref{eq:node-nonlinearity}) measures the degree that concept node $C_i$ causes concept node $C_j$ at time $t_k$:
\begin{align}
	e_{ij}  \; = \;  Degree \left(C_i \rightarrow C_j \right) \;.
\end{align}
\
These $n^2$ causal edge values define the FCM's $n \times n$ fuzzy adjacency matrix or causal edge matrix $\mathbf{E}$.
The $i$th row lists the causal edge values $e_{i1}, e_{i2}, \ldots , e_{in}$ that flow out from $C_i$ to the other concept nodes (including to itself).
The $j$th column lists the causal edge values $e_{1j}, e_{2j}, \ldots , e_{nj}$ that flow into $C_j$ from the other concept nodes.
So the $i$th row defines the causal \emph{fan out} vector of concept node $C_j$.
The $j$th column defines the causal \emph{fan in} vector of $C_j$.
The matrix diagonal lists any causal self-excitation of the $n$ concept nodes.

We can also interpret $e_{ij}$ in terms of fuzzy subsethood \parencite{kosko-fuzzyentropy1986, kosko2004}).
Then $e_{ij}$ states the degree to which the fuzzy concept set $C_i$ is a fuzzy or partial subset of fuzzy concept set $C_j$ \parencite{kosko1986}.
This abstract framework implies that the edge value $e_{ij}$ is the degree to which the fuzzy concept set $C_i$ belongs to the fuzzy power set of fuzzy set $C_j$.
This is only one interpretation of causal degrees.
A FCM simply models causality as a fuzzy signed directed edge.

Figure \ref{fig:herd-dyn-ex} shows a FCM fragment that models an undersea causal web of dolphins in the presence of sharks or other survival threats \parencite{dickerson-kosko1994}.
The next section shows how to make what-if inferences or predictions with this simple FCM that has binary concept nodes and trivalent causal edges.
The inference process uses only vector-matrix multiplication and thresholding.
More complex FCMs can activate concept nodes with nonlinear functions or with many other monotonic or nonmonotonic functions.
A causal learning law can approximate the causal edge values given time-series data of the concept nodes.

\begin{figure}
\centering
\includegraphics[height=0.25\textheight]{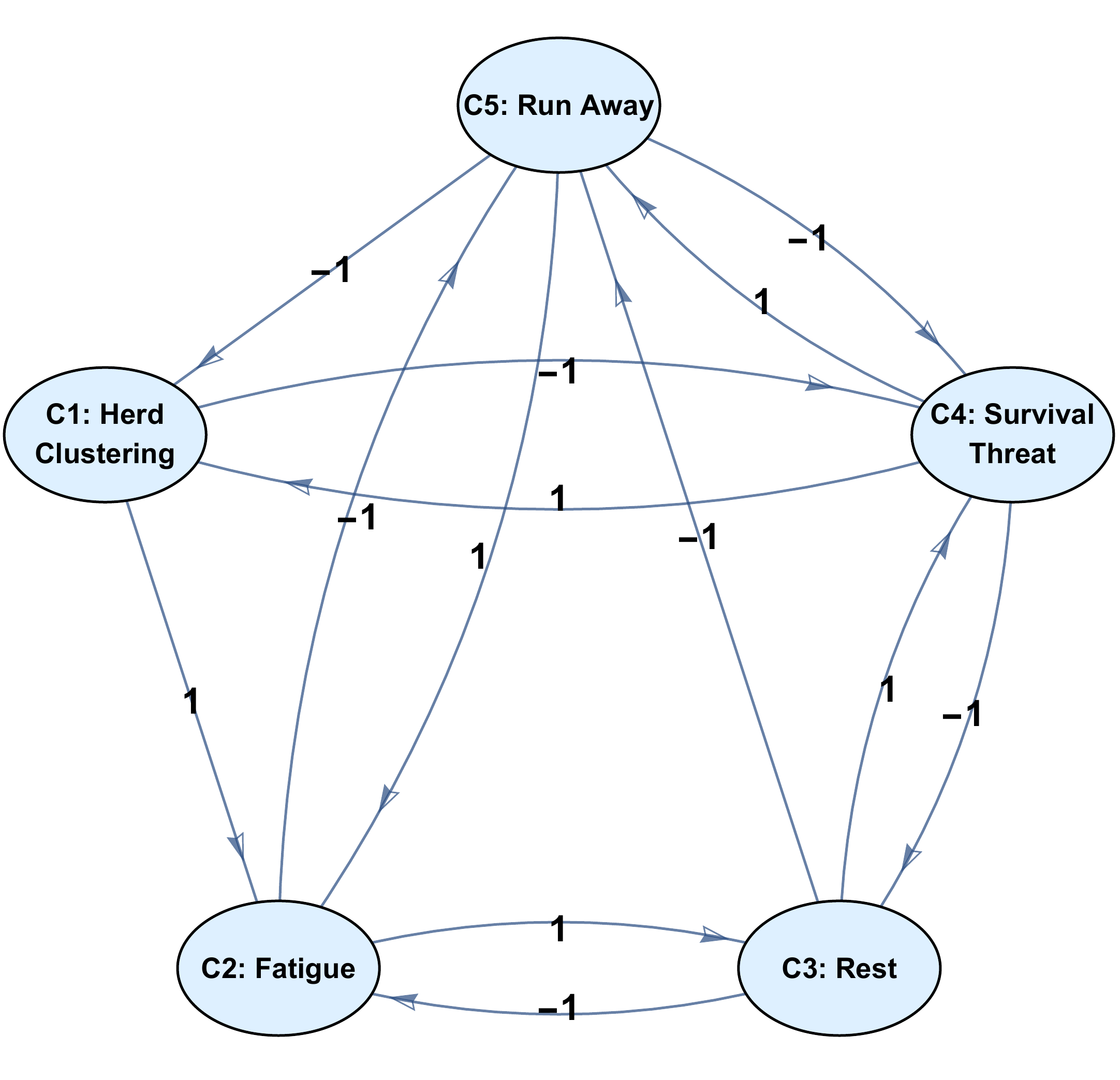}
\caption{
Fragment of a predator-prey fuzzy cognitive map that describes dolphin behavior in the presence of sharks or other survival threats \parencite{dickerson-kosko1994}. \\
\textbf{Figure Note}: The FCM itself is a fuzzy signed directed graph with feedback.
The concept nodes of the digraph represent fuzzy sets that activate to varying degrees of concept-node occurrence.
The edges denote fuzzy or partial causal dependence between concept nodes.
The edges in this FCM are trivalent:  $e_{ij} \in \{-1, 0, 1\}$.
Each nonzero edge defines a causal if-then rule:  The dolphin pod decreases its resting behavior if a shark or other survival threat is present.
But the survival threat increases if the pod rests more.
These two causal links define a minimal cycle or feedback loop within the FCM's causal web.
Such feedback cycles endow the FCM with transient and equilibrium dynamics.
All inputs produce equilibrium limit cycles or fixed-point attractors in the simplest case when all nodes are bivalent threshold functions and when the system updates all nodes at each iteration.
}
\label{fig:herd-dyn-ex}
\end{figure}

\subsection{FCM Dynamics for Binary Concept States} \label{sec:fcm-dyn}

We first illustrate FCM inference with the simple but common case of binary or on-off concept nodes.
FCM dynamics depend on the FCM's nonlinear feedback structure.
 A FCM's feedback loops model interwoven causal relationships and can produce rich and predictive equilibrium dynamics \parencite{kosko-hidden1988,  kosko-nnfs}.
These causal equilibria define ``hidden patterns'' \parencite{kosko-hidden1988} in the often inscrutable web of edges and nodes.
 FCMs with binary concepts  produce either limit-cycle equilibria or simple fixed-point attractors.
 A limit cycle is an ordered sequence of FCM states that repeats.
 A fixed point is a limit cycle of length one.
Properly fuzzy concept nodes can in principle produce more exotic equilibria such as limit tori or chaotic attractors.
See \parencite{hirsch2012} for the formal definition of these dynamical equilibria.

The long-run evolution of the FCM state vector $\mathbf{C}$
\begin{align}
\lim_{t \rightarrow \infty}\mathbf{C}(t)
\end{align}
 depends on the initial state $\mathbf{C}(0)$ as well as on the nonlinear structure of the concept nodes and the structure of the FCM causal edge matrix $\mathbf{E}$.
Simple two-state or binary-node  FCMs converge either to a fixed-point attractor
\begin{align}
	\mathbf{C}^*(t+1) = \Phi( \mathbf{C}^*(t) \; \mathbf{E})
\end{align}
or to a limit cycle of repeating bit vectors.
This convergence assumes synchronous updating of all the concept nodes at each time step.
This stability or convergence guarantee for binary-node FCMs follows from the local result that every square connection matrix  is temporally stable~\parencite{kosko-nnfs}.

Consider again the FCM fragment in Figure \ref{fig:herd-dyn-ex} that describes some of the predator-prey behavior of a dolphin pod in the presence of sharks or other survival threats \parencite{dickerson-kosko1994}.
The concept nodes are binary with threshold activations that obey (\ref{eq:threshold-node}).
Bivalent nodes simplify the dynamical analysis because updating all $n$ nodes at the same time must lead to either a fixed-point attractor or a limit-cycle of bit vectors.

Inference requires a causal edge matrix $\mathbf{E}$.
So we start with the edge matrix that underlies the dolphin FCM in Figure \ref{fig:herd-dyn-ex}.

The edges in the dolphin FCM fragment in Figure \ref{fig:herd-dyn-ex}  are trivalent:  $e_{ij} \in \{-1,0, 1\}$.
So an edge describes maximal causal increase ($e_{ij} = 1$) or maximal causal decrease ($e_{ij} = -1$) or there is no causal relationship at all ($e_{ij} = 0$).
The causal edge adjacency matrix $\mathbf{E}$ for the FCM in Figure~\ref{fig:herd-dyn-ex} is a 5-by-5 trivalent matrix:

\begin{align}
\mathbf{E}=\left(
\begin{array}{cccccc}
  & C_1 & C_2 & C_3 & C_4 & C_5 \\
 C_1 & 0 & 1 & 0 & -1 & 0 \\
 C_2 & 0 & 0 & 1 & 0 & -1 \\
 C_3 & 0 & -1 & 0 & 1 & -1 \\
 C_4 & 1 & 0 & -1 & 0 & 1 \\
 C_5 & -1 & 1 & 0 & -1 & 0 \\
\end{array}
\right)
\label{eq:fcm-mtx-ex}
\end{align}

 A key argument for using trivalent edge weights $e_{ij}$ in $\{-1, 0, 1\}$ here and elsewhere is that experts may find it hard to accurately state a graded measure of causal intensity $e_{ij} \in [-1,1]$ for a causal dependence.
 It is often much easier to elicit just sign values from experts than real-valued magnitudes.
Taber et al.~\parencite{taber-yager-helgason2007} refer to this difficulty as the expert's \emph{articulation burden}.
Real-valued magnitudes also tend to be less reliable.
Experts are far more likely to agree on edge signs than on both signs and magnitudes.
Even the same expert may state different edge-value magnitudes at different times.
This articulation burden motivates averaging the trivalent-edge-valued FCMs of experts to approximate the unknown population FCM.
We also studied a thresholded version of the Thucydides-trap FCM below to gauge the robustness of the corresponding FCM with fractional edge values.
Thresholding produced  a trivalent FCM.

The stochastic convergence result in the appendix of \parencite{taber-yager-helgason2007} shows that averaging FCMs with trivalent edges approximates the underlying population FCM that has real edge values.
FCM sample averages converge with probability one to the population average in accord with the strong law of large numbers.
The underlying limit-cycle structure of the averaged FCM also appears to approximate the limit-cycle structure of the original or population FCM if the concept nodes are binary.
The limit-cycle results in \parencite{taber-yager-helgason2007} are only preliminary simulations.

We now show how a limit-cycle hidden pattern occurs in the dolphin FCM in Figure 1.
Suppose that a shark appears at time $t = 0$.
Then the fourth or survival-threat concept node occurs or turns on.
We can represent this initial state $\mathbf{C}(0)$ of the FCM with the unit bit vector
\begin{equation}
	\mathbf{C}(0) = (0, 0, 0, 1, 0)\;.
\end{equation}
Each of the 5 concept nodes acts as a threshold function with zero threshold as in (\ref{eq:threshold-node}).
So $C_k(t) = 1$ if and only if its total inner-product input $x$ is positive:  $x > 0$.
It otherwise equals zero and thus turns off or stays off if it is not active.
Then a forward inference gives the following sequence of FCM state vectors:

\[
\begin{array}{ll}
\mathbf{C}(0) \mathbf{E} = (1, 0, -1, 0, 1)  &\longrightarrow (1, 0, 0, 0, 1) = \mathbf{C}(1) \\
\mathbf{C}(1) \mathbf{E} = (-1, 2, 0, -2, 0)  &\longrightarrow (0, 1, 0, 0, 0) = \mathbf{C}(2) \\
\mathbf{C}(2) \mathbf{E} = (0, 0, 1, 0, -1)  &\longrightarrow (0, 0, 1, 0, 0) = \mathbf{C}(3) \\
\mathbf{C}(3) \mathbf{E} = (0, -1, 0, 1, -1)  &\longrightarrow (0, 0, 0, 1, 0) = \mathbf{C}(0) \;.
\end{array}
\]

This inference sequence defines an equilibrium 4-step limit cycle because the next state vector $\mathbf{C}(4) = (0, 0, 0, 1, 0)$  is just the first state vector $\mathbf{C}(0)$.
So the FCM equilibrium or hidden pattern is the indefinitely repeating cycle $\mathbf{C}(0) \rightarrow \mathbf{C}(1)  \rightarrow \mathbf{C}(2) \rightarrow \mathbf{C}(3) \rightarrow \mathbf{C}(0) \rightarrow \cdots$.
This cycle defines the equivalent cycle of bit vectors $(0, 0, 0, 1, 0)  \rightarrow (1, 0, 0, 0, 1)  \rightarrow (0, 1, 0, 0, 0) \rightarrow (0, 0, 1, 0, 0) \rightarrow (0, 0, 0, 1, 0) \rightarrow \cdots$.
The repeating cycle predicts a predator-prey oscillation:  The shark threat appears.
Then the threatened dolphin pod clusters and runs away.
Then the dolphins get tired.
Then they rest.
But the resting dolphins then attract a shark and so on.
This limit cycle can model an incidental appearance of a shark.

Suppose instead that a shark appears and actively pursues the dolphins.
We can model this what-if policy scenario by clamping the fourth node on during each update.
This again amounts to adding a large positive input value for $I_4$ in (\ref{eq:node-nonlinearity}).
Clamping leads to two transient bit-vector states and then a stable 3-step equilibrium limit cycle:
\[
\begin{array}{ll}
\mathbf{C}(0) \mathbf{E} = (1, 0, -1, 0, 1)  &\longrightarrow (1, 0, 0, 1, 1) = \mathbf{C}(1) \; \text{  since we keep $C_4 = 1$ throughout.} \\
\mathbf{C}(1) \mathbf{E} = (0, 2, -1, -2, 1)  &\longrightarrow  (0, 1, 0, 1, 1) = \mathbf{C}(2) \\
\mathbf{C}(2) \mathbf{E} = (0, 1, 0, -1, 0)  &\longrightarrow (0, 1, 0, 1, 0) = \mathbf{C}(3) \\
\mathbf{C}(3) \mathbf{E} = (1, 0, 0, 0, 0)  &\longrightarrow (1, 0, 0, 1, 0) = \mathbf{C}(4) \\
\mathbf{C}(4) \mathbf{E} = (1, 1, -1, -1, 1)  &\longrightarrow  (1, 1, 0, 1, 1) = \mathbf{C}(5) \\
\mathbf{C}(5) \mathbf{E} = (0, 2, 0, -2, 0)  &\longrightarrow (0, 1, 0, 1, 0) = \mathbf{C}(3) \;.
\end{array}
\]

The equilibrium 3-step limit cycle is $\mathbf{C}(3) \rightarrow \mathbf{C}(4)  \rightarrow \mathbf{C}(5) \rightarrow \mathbf{C}(3) \rightarrow  \cdots$ or $(0, 1, 0, 1, 0)  \rightarrow (1, 0, 0, 1, 0)  \rightarrow (1, 1, 0, 1, 1)  \rightarrow (0, 1, 0, 1, 0)   \rightarrow \cdots$.
The limit cycle defines and thus predicts a different form of predator-prey behavior:  The shark tires the dolphin pod.
The dolphins cluster in a safety maneuver.
They then try to rest and still run away as they fatigue.
The shark does not relent and the dolphins fatigue and so on.

We can formally describe the forward spread of causal activation in a FCM through the tools of the differential calculus.
The Appendix treats the important special case where the concept nodes are soft thresholds or other differentiable functions of their inputs.
The main theorem confirms that FCM causal activation is transitive.

We next show how to combine any number of FCMs into a common FCM knowledge base through an averaging or mixing process.
The result is always some causal edge matrix $E$.
Then we will return to this fixed-matrix case and show how forward causal inference proceeds when the causal concept nodes are smoothly differentiable functions of their inputs.
We will then show how time-differentiable causal edges can define causal learning laws.


\section{Combining Causal Knowledge:  Averaging Edge Matrices} \label{sec:fcm-comb}

Causal modeling faces a threshold epistemic question when dealing with multiple experts:  How do we combine the causal models of multiple knowledge sources?

A common answer avoids the question by combining the causal knowledge or expertise before it enters a causal model.
Some form of this knowledge preprocessing occurs with AI search trees and other DAG models.
Multiple knowledge sources may lead a knowledge engineer to draw or otherwise modify a weighted causal arrow in a model.
That differs from first letting each source have its own causal arrow and then combining.
This fit-all-in-one-model approach may work well for problems of small dimension or small expert sample size.
Even then it may obscure the disparate knowledge that went into the representation.
But it can ensure that a causal DAG stays a DAG as it encodes new information.
The approach can become more ad hoc and restrictive as the expert sample size $m$ grows.
The likelihood of getting a causal cycle only increases with expert sample size and the node count of the model.

FCMs answer the epistimic question directly:  Average the causal FCMs of each expert \parencite{kosko-hidden1988, kosko-nnfs, taber1991}.
Preprocessing can still occur.
But there is no limit to the number $m$ of FCMs that averaging can combine.
The result is always a FCM and one with all the representative properties of a sample average.
This numerical result holds even though experts may express their knowledge solely in words.
Laws of large numbers can apply directly or partially based on how well the $m$ expert's FCM sample approximates a random sample.

The FCM average forms a mixture or convex combination of the causal edges.
A group of $m$ experts can each produce an FCM causal edge matrix $\mathbf{E}_k$ that describes some fixed problem domain.
Each expert can model different concept and policy nodes.
The total number of nodes is $n$.
Augment the edge matrices with zero rows and columns for any missing nodes in an expert's causal edge matrix.
Then FCM knowledge fusion or combination takes the weighted average of their augmented causal edge matrices:
\begin{align}
 \bar{\mathbf{E}}_m = \sum_{k=1}^m  w_k \mathbf{E}_k
 \label{eq:fcm-sum}
\end{align}
where the weights $w_k$ are convex weights and hence nonnegative and sum to one.

The weights $w_k$ can reflect relative expert credibility in the problem domain.
They can reflect test scores or subjective rankings or some other measure of the experts' predictive accuracy in prior experiments.
The same weight $w_k$ need not apply to the entire $k$th FCM edge matrix.
Each edge value can have its own weight.
So a weight matrix $W_k$ corresponds to each expert's FCM edge matrix.
Predd et al.~\parencite{predd-et-al2008} developed a method for combining expert inputs when the experts abstain or when they are incoherent.
Voting schemes \parencite{conitzer-et-al2009, caragiannis-procaccia2013} might also pick the FCM weights and affect the fusion process.
We here take the weights as given and use equal weights as a default.

The $m$ edge matrices $\mathbf{E}_k$ in (\ref{eq:fcm-sum}) must be conformable for addition.
So they must have the same number of rows and columns and in the same matrix positions.
So we first take the union of all concept nodes from all $m$ knowledge sources.
This again gives a total of $n$ distinct concept nodes.
Then we zero-pad or add rows and columns of zeros for missing nodes in a given knowledge source's causal edge matrix.
This gives a conformable $n$-by-$n$ signed fuzzy adjacency matrix $\mathbf{E}_k$ after permuting rows and columns to bring them in mutual coincidence with all the other zero-padded augmented matrices.

The strong law of large numbers gives some guarantees about the convergence of this fusion knowledge graph to a representative population FCM if the knowledge sources are approximately statistically independent and identically distributed and if they have finite variance \parencite{kosko-hidden1988, taber-yager-helgason2007}.
Then the weighted average in (\ref{eq:fcm-sum}) can only reduce the inherent variance in the expert sample FCMs.
So the knowledge fusion process improves with sample size $m$.
Simulations have shown that the equilibrium limit cycles of the combined FCM tend to resemble the limit cycles of the $m$ individual FCMs \parencite{taber-yager-helgason2007}.
An expert random sample is sufficient for this convergence result but not necessary.
A combined FCM may still give a representative knowledge base when the expert responses are somewhat correlated or when the experts do not all have the same level of expertise or problem-domain focus.
Users can also use policy articles or books or legal testimony as proxy experts.

Figure~\ref{fig:fcm-combo} shows the minimal combination case where two FCMs fuse into one representative FCM.
The mixture or convex combination of FCMs creates a new fused FCM as the weighted averages of the FCMs' augmented signed fuzzy adjacency matrices.
Users can add new concept nodes or factors at will.
Each new factor converts all $m$ $n$-by-$n$ edge matrices into $n+1$-by-$n+1$ edge matrices.
This again amounts to adding a new zero-padded row and column to an edge matrix if its corresponding FCM does not include the factor as a concept node.
An expert has a zero row and column for a concept node if the expert impliedly states that that concept is not causally relevant.

This fusion averaging technique can reflect bad effects as well as any other effect.
The technique can reflect anomalous effects due to active sabotage or extreme variance in expert opinions.
Highly variable expert inputs will tend to produce a highly variable FCM causal knowledge base.
There may be no benefit from combining expert edge values that approximate thick-tailed probability densities.
Cauchy probability bell curves closely resemble normal probability bell curves.
Cauchy bell curves have slightly thicker tails that give rise to far more variable realizations.
But the sample average of Cauchy random variables is itself a Cauchy random variable.
So there is no benefit or decrease in system variance whatsoever in this thick-tailed case.
The combined result has the same infinite variance that any one of the individual Cauchy samples has.
Combining knowledge sources with even thicker-tailed probability densities (such as many alpha-stable densities) can produce variability even more extreme than the variability of any of the combined knowledge sources.

Large-scale FCM combination can combine FCMs with simple recursive updates.
Then the new FCM edge matrix equals the current combined FCM edge matrix plus a correction term that includes the new FCM edge matrix.
This recursive form of FCM combination can assist large-scale online FCM knowledge combination in social media and elsewhere.
The recursions apply locally to the current combined edge value $e_{ij}(m)$ that gives the $ij$th causal edge of the $m$ combined FCMs $F_1, \ldots, F_m$.

We state the recursions for the simplest case where all FCMs and hence all knowledge sources have the same weight or credibility.
Let $\bar{e}_{ij}(m)$ denote the sample mean of the first $m$ causal edges $e_{ij}(1), \ldots, e_{ij}(m)$:
\begin{align}
	     \bar{e}_{ij}(m)= \frac{1}{m} \sum_{k=1}^m e_{ij}(k)  .     \label{eq:FCM-edge-sample-mean}
\end{align}
Let $S_{e_{ij}}^2(m)$ denote the corresponding unbiased sample variance of the first $m$ combined $ij$th edge values:
\begin{align}
	    S_{e_{ij}}^2(m)  = \frac{1}{m-1} \sum_{k=1}^m  (e_{ij}(k) -  \bar{e}_{ij}(m))^2      \label{eq:FCM-edge-sample-variance}
\end{align}
for $m > 1$.
The question is how to recursively update each of these averages given a new $ij$th edge value $e_{ij}(m + 1)$.
The answer comes from the predictor-corrector form of updates often found in Kalman filters \parencite{meditch1969stochastic}:
\begin{align}
	     \bar{e}_{ij}(m + 1)   =     \bar{e}_{ij}(m)   +  \frac{1}{m+1}[e_{ij}(m+1) -   \bar{e}_{ij}(m) ]    \label{eq:FCM-edge-sample-mean-recursion}
\end{align}
\begin{align}
	S_{e_{ij}}^2(m + 1)   =    S_{e_{ij}}^2(m)  +  \frac{1}{m}[(e_{ij}(m+1) - \bar{e}_{ij}(m))(e_{ij}(m+1) - \bar{e}_{ij}(m + 1))    -   S_{e_{ij}}^2(m)]  .
\end{align}
The new $ij$th edge statistic equals the old or predicted value plus a new or corrector value.
A similar recursion holds for updating the combined edge's sample covariance.
More complex recursions hold for variable-weight edge values although these will likely not apply in large-scale online settings.

\begin{figure*}
\centering
\includegraphics[width=0.95\textwidth]{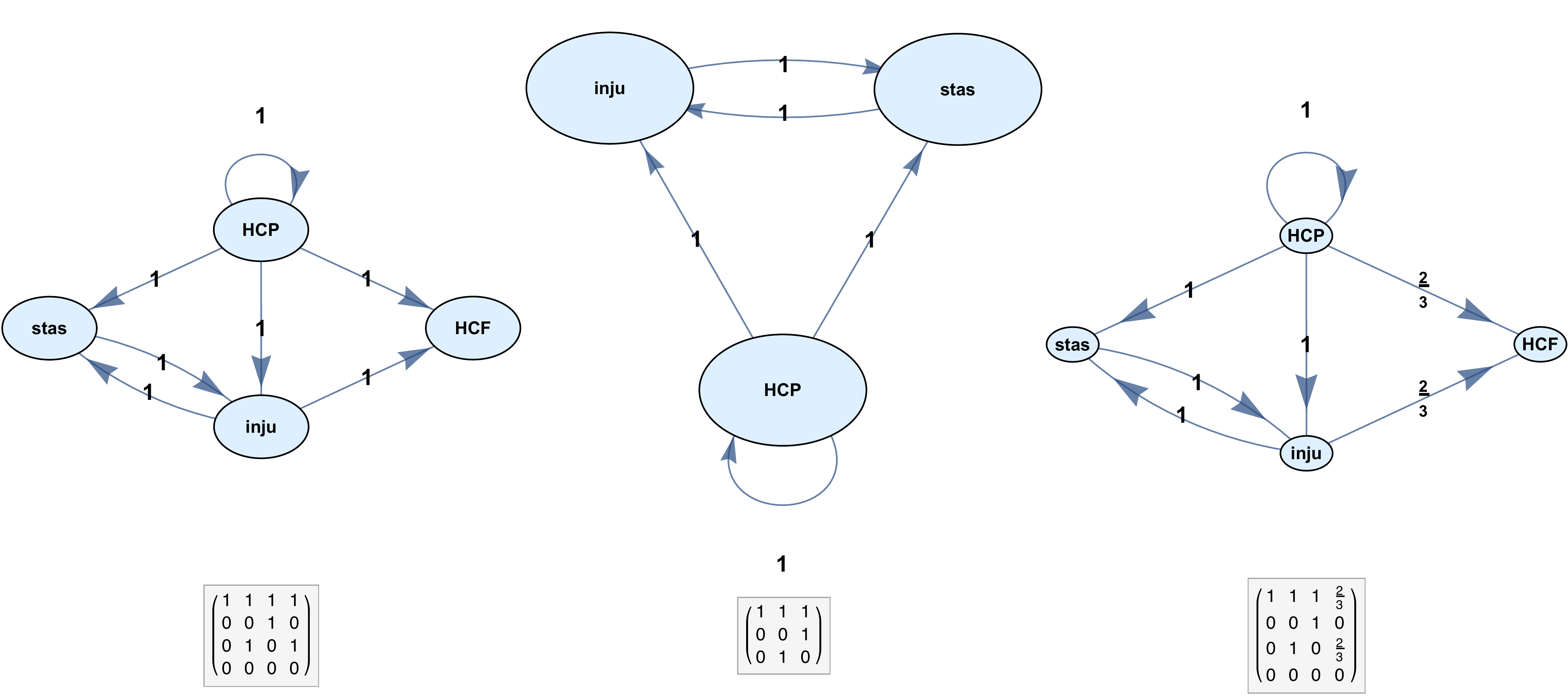}
\caption{FCM knowledge combination or fusion by averaging weighted FCM adjacency matrices. \\
\textbf{Figure Note}: The three digraphs show the minimal case of combining two FCMs that have overlapping concept nodes.
The third FCM is the weighted combination of the first FCMs.
The expert problem domain is the medical domain of strokes and blood clotting involved in Virchow's Triad of blood stasis $stas$, endothelial injury $inju$, and hypercoagulation factors $HCP$ and $HCF$~\parencite{taber-yager-helgason2007}.
Expert $1$ has a larger FCM than Expert $2$ has because Expert $1$ uses an extra concept node.
FCM knowledge combination fuses their knowledge webs by averaging their causal-edge adjacency matrices with the mixing equation (\ref{eq:fcm-sum}).
This weighted average uses each expert's causal edge matrix $\mathbf{E} = \frac{2}{3}\mathbf{E}_1 + \frac{1}{3}\mathbf{E}_2$ as shown in the combined (third) FCM.
The weighting assumes that the first expert is twice as credible as the second expert.
Note that Expert $2$ ignored the $HCF$ factor.
This results in a new row and column of all zeroes in the augmented edge matrix $\mathbf{E}_2$.
}
\label{fig:fcm-combo}
\end{figure*}

We turn next to inferring the directed causal edges $e_{ij}$ from time-series data.

\section{Learning FCM Causal Edges}\label{sec:fcm-learn}

Correlation does not imply causation.
But some time-lagged correlations \emph{suggest} causation.
This is the idea behind the \emph{unsupervised} learning laws below for estimating the directed causal edges $e_{ij}$.
They are unsupervised because there is no teaching signal that the learning process matches against.

We can learn causal edge strengths through the \emph{concomitant activation} among the factor pairs.
This approach assumes that events (factor activities) are more likely to involve a causal connection if the events occur together \parencite{hebb1949, kosko-nnfs, kosko-DHL1986}.
This suggests the well-known Hebbian correlation learning law (neurons that fire together wire together) for training neural network synaptic weights \parencite{kosko-nnfs}:
\begin{align}
\dot{e}_{ij} = -e_{ij} + {C}_{i}{C}_{j}
\label{eq:hebbian}
\end{align}
where $\dot{x}$ denotes the time derivative of the signal $x$.
The passive decay term $-e_{ij}$ stabilizes the learning in the differential-equation model.
It also models a ``forgetting'' constraint that helps the network prune inactive connections.
The product term ${C}_{i}{C}_{j}$ directly models concomitant correlation.

We can instead use \emph{concomitant variation}~\parencite{jsm1843} in time between factors as partial evidence of a causal relation between those factors or concepts.
Suppose the data show that an increase in $C_i$ occurs at the same time as increase in the $C_j$.
This concomitant increase suggests that the edge value $e_{ij}$ should be positive.
Suppose similarly that decreases in $C_i$ occur with decreases in $C_j$.
Then such concomitant decrease suggests a negative causal edge value $e_{ij}$.
Even a slight time lag can between the two concept nodes can indicate the direction of causality in practice.
Such concomitant variation or covariation leads to the \emph{differential} Hebbian learning (DHL) law~\parencite{kosko-DHL1986, kosko-nnfs, kosko-hidden1988}:
\begin{align}
\dot{e}_{ij} = -e_{ij} + \dot{C}_{i}\dot{C}_{j} \;.
\label{eq:dhl}
\end{align}

We use concomitant activation and variation as proxies for causation during unsupervised learning with Hebbian and differential Hebbian learning Laws.
Hebbian learning tends to learn spurious causal links between any two concept nodes that occur at the same time.
This quickly grows an edge matrix of nearly all unity values if most of the nodes are active.
DHL correlates node velocities.
So it has a type of arrow of time built into it.
DHL correlates the signs of the time derivatives.
So it grows a positive causal edge value $e_{ij}$ if and only if the concept nodes $C_i$ and $C_j$ both increase or both decrease.
 It grows a negative edge value if and only if one of the nodes increases and the other decreases.

Both learning laws combine to give a more general version of DHL \parencite{kosko-hidden1988}:

\begin{align}
\dot{e}_{ij} = -e_{ij} + {C}_{i}{C}_{j} + \dot{C}_{i}\dot{C}_{j} \;.
\label{eq:hl+dhl}
\end{align}

This hybrid learning law fills in expected values for edge-strength vales when there is no signal variation in the factor set~\parencite{kosko1990unsupervised}.
The hybrid law takes advantage of the relatively rarer variation events to update the edge weights.
It also tends to produce limit cycles or even more complex equilibrium attractors.
It can produce fixed-point attractors given some strong mathematical assumptions \parencite{kosko-hidden1988, kosko-nnfs}.

Most applications use discretized versions of the DHL law~\parencite{kosko-fuzeng} in (\ref{eq:dhl}):

\begin{align}
{e}_{ij}(t+1) =
\begin{cases}
e_{ij}(t) + \mu \left[ \Delta {C}_{i}(t) \Delta {C}_{j}(t) - e_{ij}(t) \right] &\; \Delta C_i(t) \neq 0\\
e_{ij}(t) &\; \textrm{else}   \label{eq:discrete-dhl-law}
\end{cases}
\end{align}
where $\Delta C_k(t) = C_k(t) - C_k(t-1)$.

 DHL can infer causal edge values in a FCM if the system has access to enough time-series data.
 Such data can again come from expert opinion surveys.
 It can come from direct time-series data on measurable factors.
 Or it can come from indirect instrumental variables linked to the factors of interest:  social media trends, Google trends, or topic modeling on news corpuses.
  Figures~\ref{fig:fcm-adapt} and \ref{fig:blm} show DHL training paths for single causal edge values.
   Figure~\ref{fig:fcm-adapt} learns a causal edge for the PSOT FCM in the next section.
   Figure~\ref{fig:blm} shows the DHL training of an edge value using Google Trends time-series data of the use of politically charged terms ``Black lives matter,'' ``All lives matter,'' and ``Blue lives matter'' in online discourse.
 DHL here converged to an approximation of the causal edge values after only a few iterations.

We can also fuse soft and hard knowledge sources through the above averaging technique in (\ref{eq:fcm-sum}).
Let $\mathcal{E}_{data}$ denote the data-driven FCM.
Let $\mathbf{E}_{exp}$ denote the expert-elicited FCM.
Then the fused causal edge matrix $\mathbf{E}_{fusion}$ is a simple mixture of the two edge matrices:
\begin{equation}
	\mathbf{E}_{fusion} = \omega_{data} \mathbf{E}_{data} + \omega_{exp} \mathbf{E}_{exp} \;.
\end{equation}
Then (\ref{eq:discrete-dhl-law}) or some other statistical learning law can continue the adaptation process by using new numerical data or occasional opinion updates from experts.

We can also learn edge values by taking a cue from the literature on Bayesian networks~{\parencite{friedman-koller2003}}.
This entails putting a prior on a randomized FCM.
Assume first that the FCM graph is random.
Assume next a prior over the space of amenable graphs.
Then use observed node data to update a posterior distribution of compatible FCM graphs.
This Bayesian process requires taking care with the topology and size of the graph spaces.
The process also requires that the user produce an accurate and tractable closed-form prior for the graphs.
Fuzzy rules can directly represent these closed-form priors \parencite{osoba-mitaim-kosko-SMC2011}.

\begin{figure}
\centering
\includegraphics[width=4in]{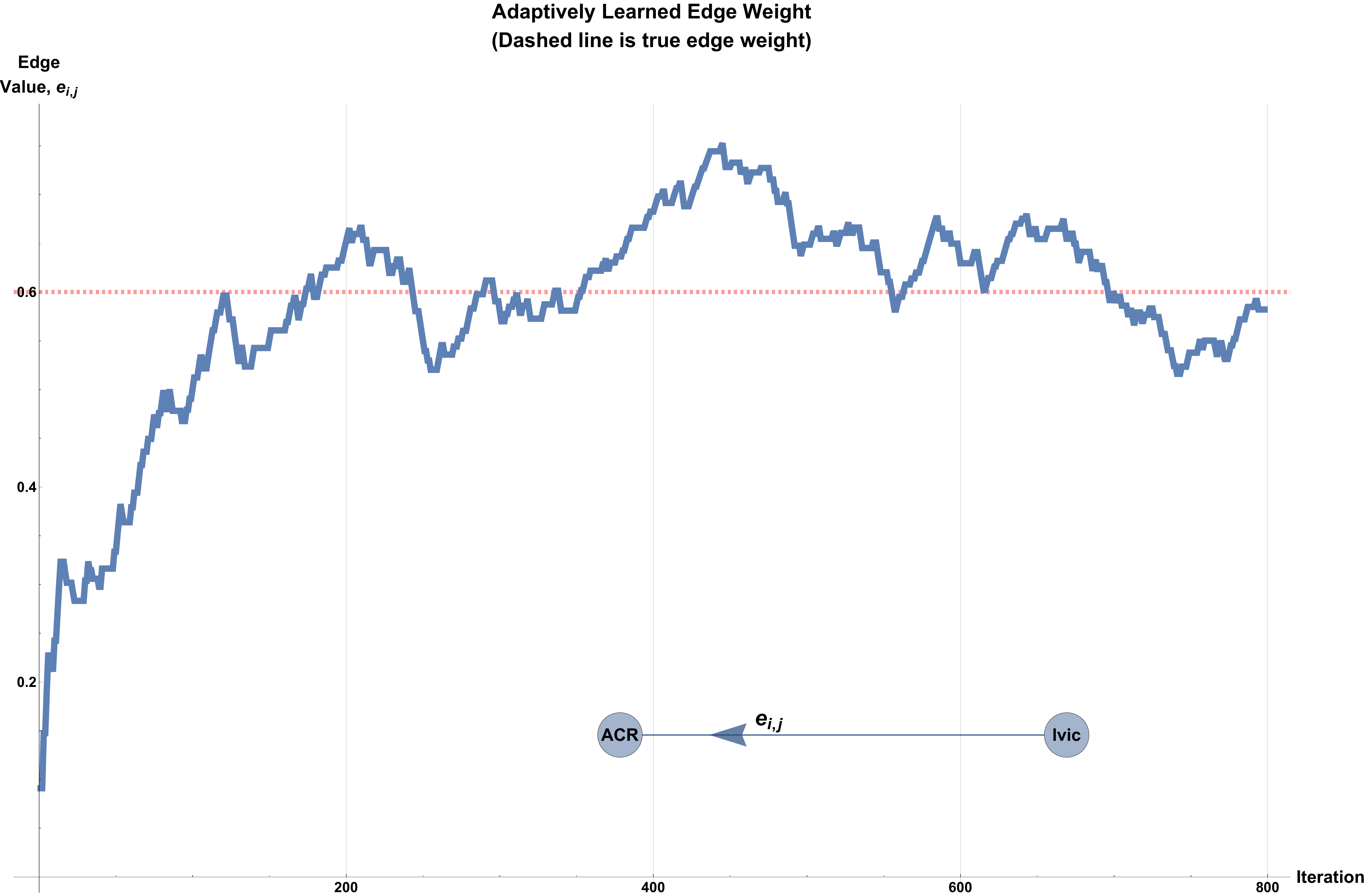}
\caption{
Learning FCM causal-edge values $e_{ij}$ with time-series data from activated causal concept nodes. \\
\textbf{Figure Note}: We can infer the value of this directed causal edge with adaptive inference algorithms such as differential Hebbian  learning if the system has access to enough time-series data for both concept nodes.
The time-series data may come from survey data or field measurements or expert elicitations.
This data came from two samples nodes from the PSOT model.
  The plot shows that differential Hebbian learning quickly converged to a good approximation of the edge value $e_{ij}$.
}
\label{fig:fcm-adapt}
\end{figure}

Learning need not take place only in a stationary causal environment where the underlying causal relations do not change in time.
Causal relations are apt to change in large-scale problems of social science.
Figure~\ref{fig:blm} gives an example.
Adaptive FCMs can still model these nonstationary causal worlds if the causal world does not change too fast and if the FCM learning system has access to enough time-series data that reflects these changes.

\begin{figure}[!ht]
\centering
\begin{minipage}[r]{\textwidth}
\begin{centering}
\begin{tabular}{cc}
\includegraphics[width=0.65\columnwidth,clip]{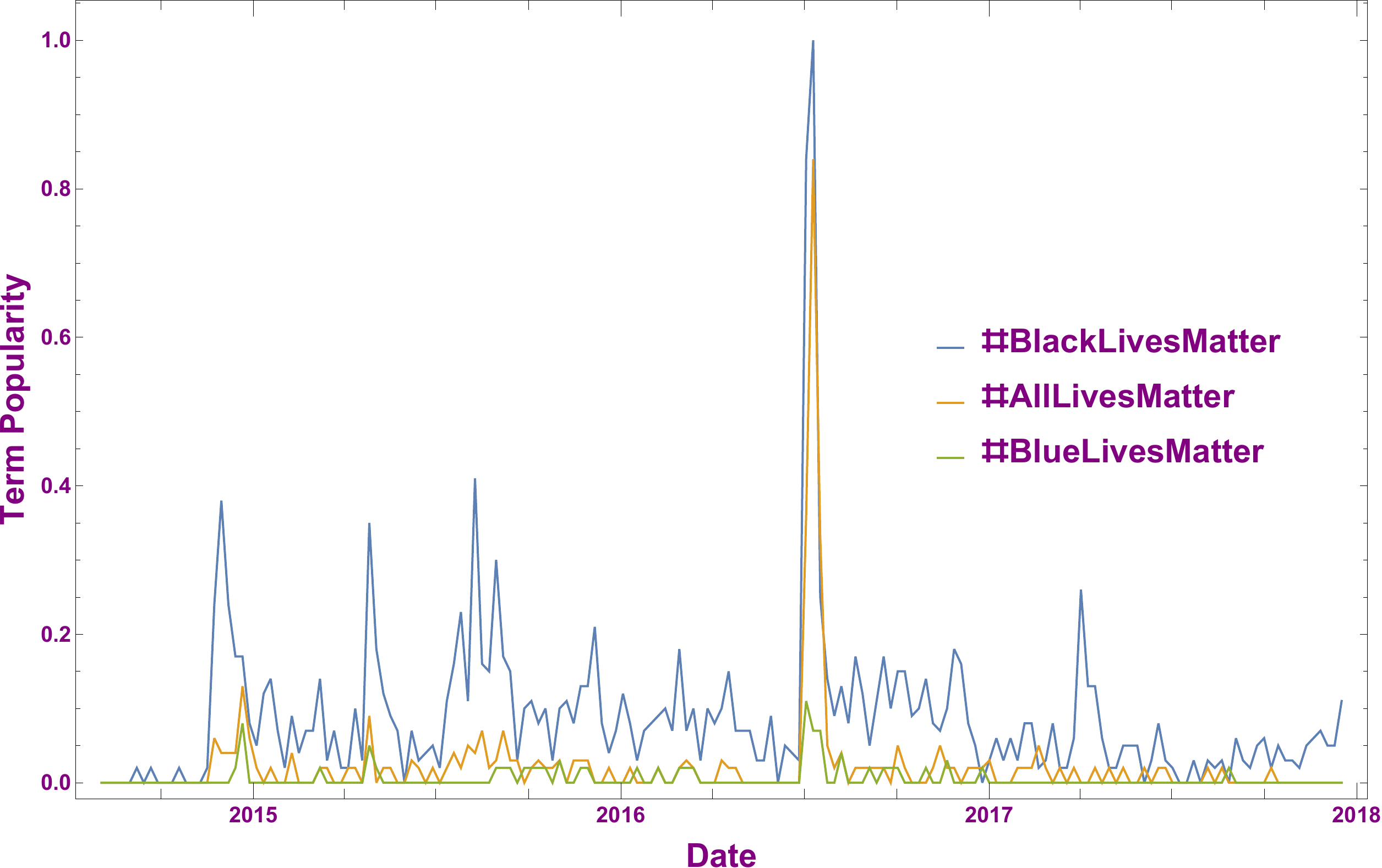}
&
\includegraphics[width=0.35\columnwidth,clip]{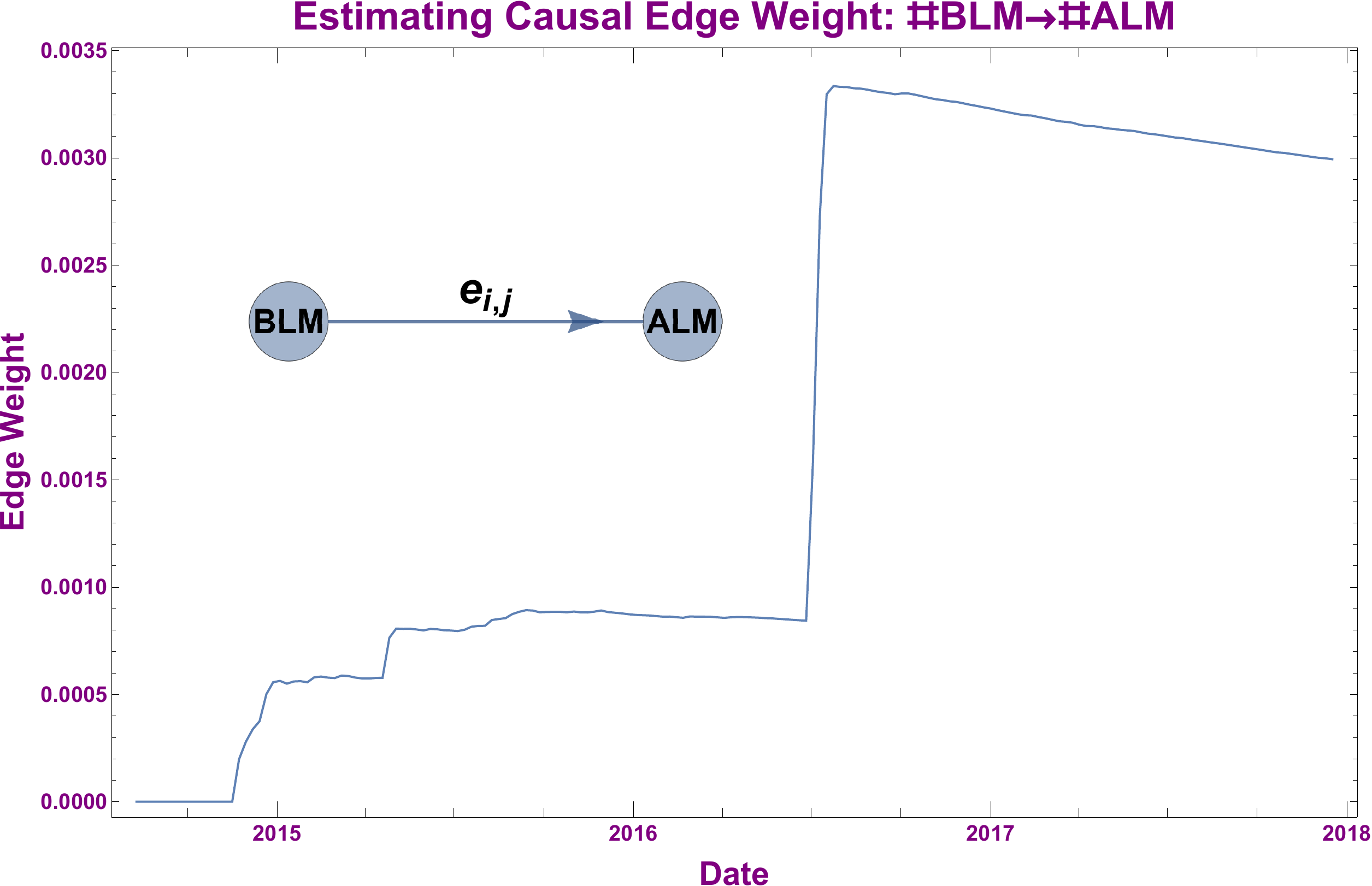}
\end{tabular}
\end{centering}
\end{minipage}
\caption{
Learning FCM causal-edge values $e_{ij}$ with Google Trends time-series data for three politically charged phrases:  Black lives matter, All lives matter, and Blue lives matter. \\
\textbf{Figure Note}: Google Trends time-series data recorded the weekly popularity of these terms in public Google-search activity from January 2014 to February 2017.
The time series consisted of $163$ ordered samples.
The use of BLM-related terms preceded the use of ALM-related terms both in time and in the media narrative.
We used this fact to specify the direction for the causal edge.
}
\label{fig:blm}
\end{figure}

\section{FCM Example:  Public Support for Insurgency and Terrorism}\label{subsec:psot}

Our first substantive FCM policy example is to the problem of public support for insurgency and terrorism (PSOT).
We based two PSOT FCMs on the factor-tree PSOT analysis of \parencite{aom-pkd2013, davis-larson2012}.
Public support for insurgency and terrorism has complex socio-political causes~\parencite{snow-soule-kriesi2008, ibrahim2007, nawaz2015, davis-cragin2009, davis-larson2012} that involve numerous factors.
 Davis's later work~\parencite{davis-wsc2015, davis-rand2015} used the PSOT model to motivate related models of an individual propensity for terrorism.

The PSOT model is a causal factor-tree model because it depicts the degree to which child nodes influence or cause parent nodes.
Figure~\ref{fig:psot} and Table~\ref{tab:psot-desc} give more details of the PSOT factor tree.
The PSOT nodes represent factors that directly or indirectly relate to the \emph{Public Support for Insurgency and Terrorism} concept $PSOT$.

\begin{figure*}
\centering
\includegraphics[width=0.8\textwidth, keepaspectratio]{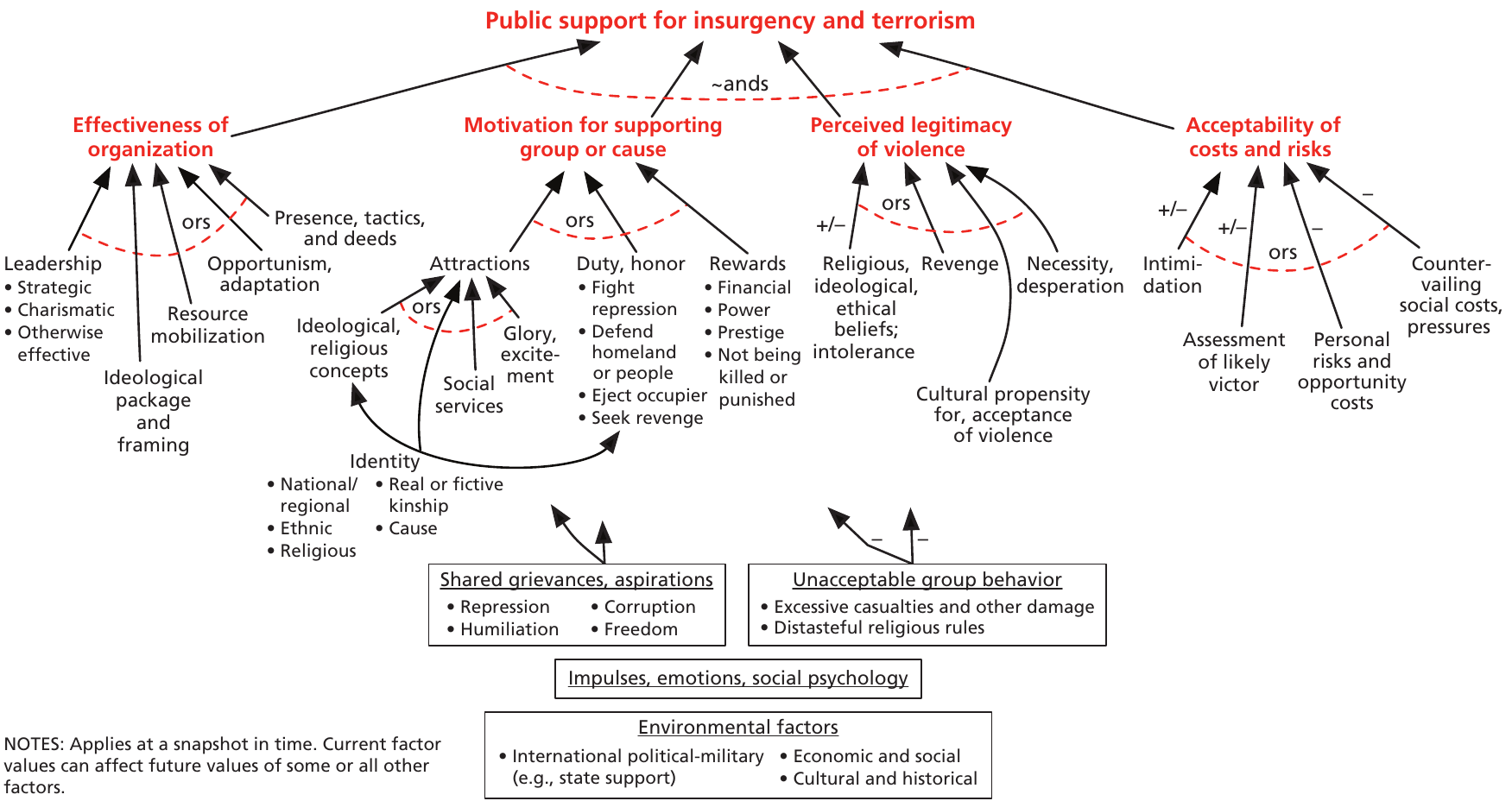}
\caption{PSOT factor-tree model.
The figure shows the directed relationships among factors that underly the Public Support for Terrorism or Insurgency model in \parencite{aom-pkd2013}.}
\label{fig:psot}
\end{figure*}

Davis's factor-tree models are multi-resolution models~\parencite{davis-bigelow1998}.
Major elements have a hierarchical structure that allows users to specify factors at different levels of detail.
Each node is an exogenously driven factor or it fires or activates based on a function of its inputs.

There are also cross-cutting factors besides sub-node factors.
Cross-cutting factors affect multiple factors simultaneously.
The ``$\tilde and$'' nodes depend on all fan-in factors being present to a first approximation.
The ``$\tilde or$'' nodes depend on any of the fan-in factors being active or on a combination of the fan-in factors being active.
There are several top-level factors that directly relate to the general $PSOT$  of \parencite{davis-larson2012}:
{Effectiveness of the organization $EFF$, motivation for supporting the group/cause $MOTV$, the perceived legitimacy of violence $PLEG$,} and {the acceptability of costs and risks $ACR$}.
Each of these factors have attendant contributory sub-factors.

PSOT edges denote positive influences by default.
We denote negative edges with `-' as with a FCM causal-decrease edge.
Factor activation along a negative edge reduces the activation of the parent factor.
We denote \emph{ambiguous} edges with ``$+/-$''.
The ambiguity refers to uncertainty over the edge's direction of influence.

We based our FCM models on the important case of the al-Qa'ida transnational terrorist organization.
We augmented the original PSOT with cross-links in the dynamic model to allow richer representation of system dynamics.

Davis et al.~\parencite{davis-larson2012} have discussed how the PSOT model explains the public support for al-Qa'ida's mission as follows (paraphrased from \parencite{davis-larson2012}).
 The organizational effectiveness of al-Qa'ida depends in part on the charisma, strategic thinking, and organizational skills of its {leadership} ($lead$).
 al-Qa'ida has framed its ideology to appeal to many Muslims worldwide.
 Motivation for public support of al-Qa'ida's beliefs comes from shared religious beliefs that stress common {identity} ($id$) and the sense of {duty} ($duty$) that such identity fosters.
  al-Qa'ida also relies on a popular narrative of {shared grievances} ($shgr$) in the Muslim world.
  al-Qa'ida stresses the perceived {glory} ($glry$) of supporting a cause that aims to redress these purported grievances.
  {Religious beliefs and intolerance} ($intl$) help increase the {perceived legitimacy} ($PLEG$) of violence against the West and against the many Muslims who do not share their Salafist views.
  {Countervailing pressure} ($scst$) discourages more support for al-Qa'ida.
  This countervailing pressure occurs in part because much of the public believes that al-Qa'ida will not succeed and thus emerge as ultimate victors ($lvic$).
  This pressure reduces the {acceptability of costs and risks} ($ACR$) for al-Qa'ida activities.
  The parameters of this al-Qa'ida case study determined the relative causal edge weights in our FCM models.

\begin{table}
\centering
\begin{tabular}{|l|l|}
\hline
\textbf{Label} & \textbf{Full Description} \\
\hline
 \text{lead} & \text{Leadership Strategic or otherwise} \\
 \text{pkg} & \text{Ideological Package $\&$ Framing} \\
 \text{rsrc} & \text{Resource Mobilization} \\
 \text{opp} & \text{Opportunism $\&$ Adaptation} \\
 \text{pres} & \text{Presence, Tactics, $\&$ Deeds} \\
 \text{EFF} & \text{Effectiveness of Organization} \\
 \text{reli} & \text{Ideological Religious Concepts} \\
 \text{socs} & \text{Social Services} \\
 \text{glry} & \text{Glory, Excitement} \\
 \text{ATT} & \text{Attractions} \\
 \text{duty} & \text{Duty $\&$ Honor} \\
 \text{rwrd} & \text{Rewards} \\
 \text{MOTV} & \text{Motivation for Supporting Group, Cause} \\
 \text{intl} & \text{Religious, Ideological, Ethical Beliefs; Intolerance} \\
 \text{rvng} & \text{Revenge} \\
 \text{cprop} & \text{Cultural Propensity for Accepting Violence} \\
 \text{desp} & \text{Desperation, Necessity} \\
 \text{PLEG} & \text{Perceived Legitimacy of Violence} \\
 \text{intm} & \text{Intimidation} \\
 \text{lvic} & \text{Assessment of Likely Victor} \\
 \text{prsk} & \text{Personal Risk and Opportunity Cost} \\
 \text{scst} & \text{Countervailing Social Costs $\&$ Pressures} \\
 \text{ACR} & \text{Acceptability of Costs $\&$ Risks} \\
 \text{id} & \text{Identity} \\
 \text{shgr} & \text{Shared Grievances $\&$ Aspirations} \\
 \text{ugb} & \text{Unacceptable Group Behavior} \\
 \text{env} & \text{Environmental Factors} \\
 \text{impl} & \text{Impulses, Emotions, Social Psychology} \\
 \text{hsucc} & \text{History of Successes} \\
 \text{mgtc} & \text{Management Competence} \\
 \text{prop} & \text{Propaganda, Advertising} \\
 \text{efdoc} & \text{Effectiveness of Indoctrination/Passing Beliefs} \\
 \text{hfail} & \text{History of Failures} \\
 \text{PSOT*} & \text{Public Support for Insurgency and Terrorism} \\
 \hline
\end{tabular}
\caption{Table of factors in the Public Support for Insurgency and Terrorism (PSOT) model.}
\label{tab:psot-desc}
\end{table}

\begin{figure*}
\centering
\includegraphics[width=0.75\textwidth]{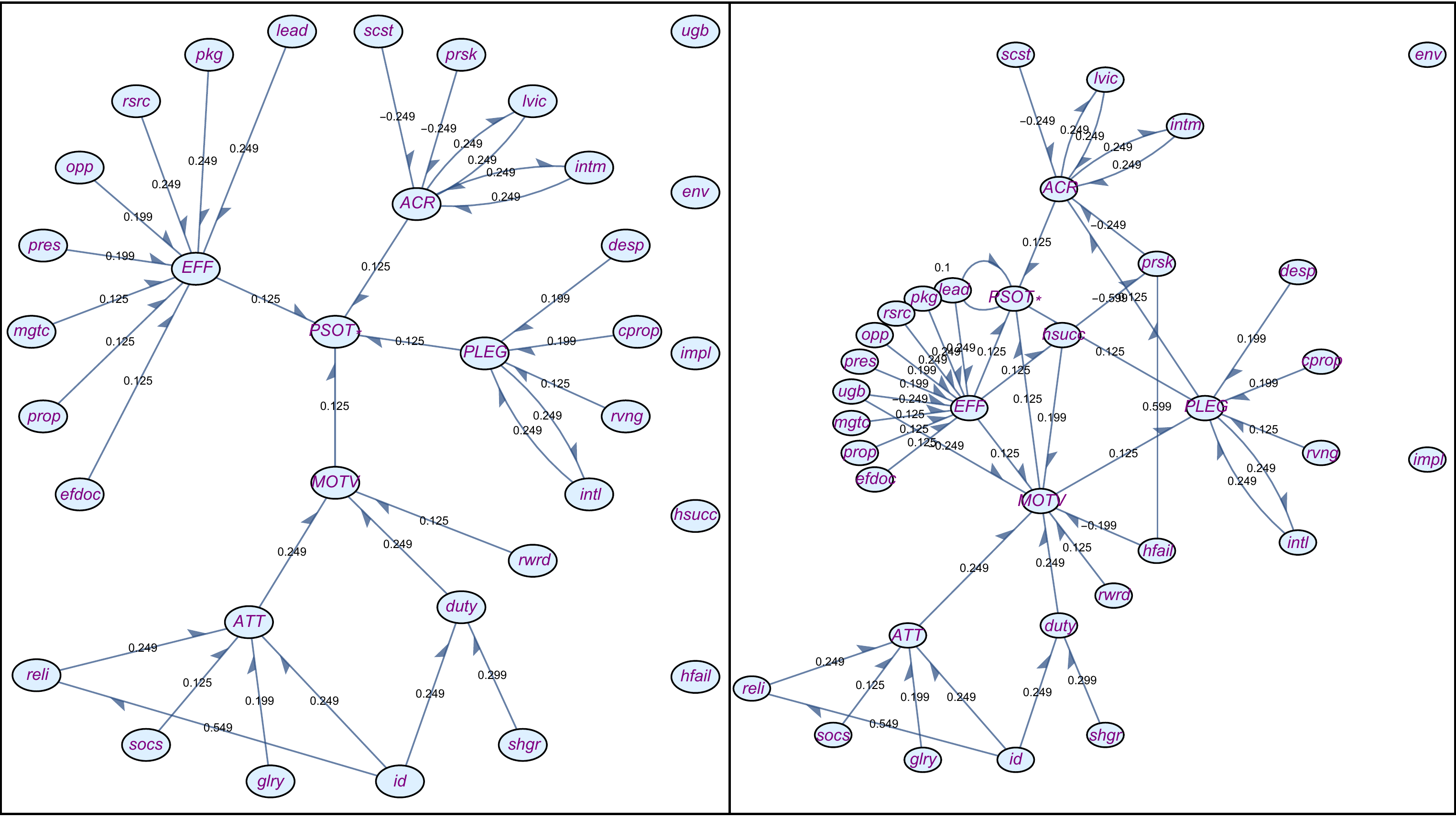}
\caption{
Two Fuzzy Cognitive Maps of the PSOT factor-tree model. \\
\textbf{Figure Note}: The left panel shows the FCM digraph for the original static (acyclic) PSOT model.
The right panel shows the FCM digraph of the dynamic PSOT model with cross-links.
Table~\ref{tab:psot-desc} gives the key for the concept node labels in both FCMs.
We based the new FCM edges in the right digraph on the findings in ~\parencite{aom-pkd2013} and on expert input.
Figure 6 in \parencite{osoba-kosko2017} illustrates the causal-edge connection matrices $\mathbf{E}$ (the adjacency matrices) for each FCM using heatmaps.
}
\label{fig:psot-compare}
\end{figure*}

Figure \ref{fig:psot-compare} shows FCM versions of the old static (acyclic) PSOT model and new dynamic PSOT model.

We now outline these changes to the original PSOT model.
We first added a weak self-excitation feedback loop to the PSOT concept node because it is the highest-level concept node.
This self-excitation loop modeled inertia in aggregate public opinion about insurgency and terrorism.
This new feedback source induced a weak serial correlation in time in the PSOT concept node.

The next directed edges connected the top-level factors in  \parencite{aom-pkd2013} from left to right:  $EFF \rightarrow MOTV$, $MOTV \rightarrow PLEG$, and $PLEG \rightarrow ACR$.
These directed causal edges made explicit an implicit point about O'Mahony and Davis's use of factor trees.
Their factor-tree representation assumed a left-to-right dependence of the top-level factors that we have linked~\parencite{davis2011,aom-pkd2013}.
This implicit dependence made their factor tree more readable.
The FCM model made this dependence explicit.

O'Mahony and Davis~\parencite{aom-pkd2013} discuss other dynamic augmentations to the PSOT model.
They point to the following new factors.
A history of successes or failures can affect motivation and perceived risks.
We model this dependence with the two factors ``history of successes'' and ``history of failures.''
These two nodes exert opposing influence on $MOTV$ and $prsk$.
We split this history factor because traditional FCM models admit only positive values that represent the degree or intensity to which a concept occurs.
The effectiveness of the organization factor \emph{EFF} partly determines the history of successes:  $EFF \rightarrow hsucc$.
Unacceptable group behavior $ugb$ also influences motivation and effectiveness: and $ugb \rightarrow MOTV$ and  $ugb \rightarrow EFF$.


\section{US-China Relations: A FCM of Allison's Thucydides Trap}\label{subsec:trap}

We next use a FCM to model a new conflict dynamic in international relations.

Political scientist Graham Allison calls this dynamic the \emph{Thucydides trap} ~\parencite{allison2015, allison2017}.
Allison argues that this dynamic occurs when a new power emerges that challenges the dominance of an older power on the world stage.
Superpowers such as the United States and China must avoid the Thucydides trap to avoid war.

Our FCM interpretation of Allison's analysis predicts some type of war pattern in some cases and not in others.
A large percentage (most) of the clamped input states led to a war-type outcome.
But this was not a probability estimate.
It reflects an exhaustive search of all possible clamped input states.
It does not reflect that relative likelihood of the clamped input states themselves.

We based the fractional causal edge values $e_{ij}$ for this FCM on Allison's text.
See the tables of textual justifications below.
We also tested the robustness of this properly fuzzy FCM by thresholding all positive edge values $e_{ij} > 0$ to $1$ and all negative edge values $e_{ij} < 0$ to $-1$.
This gave a trivalent FCM that predicted some type of war for the majority of all clamped input states.
We stress again that the prevalence of a war outcome in this model does \emph{not} mean that the FCM predicts war with high probability.
That would require that all input states are equally likely and they clearly are not.
We did not address the issue of which inputs are more or less likely to occur.
Our task was to translate Allison's textual claims into a representative FCM causal model and explore its pattern predictions.

The name ``Thucydides trap'' stems from a famous political conjecture in Thucydides' \emph{History of the Peloponnesian War}~\parencite{thucydides-jowett} (Book 1, paragraph 23): ``the real though unavowed cause [of the war] I believe to have been the growth of Athenian power, which terrified the [Spartans] and forced them into war.''
Thucydides expands on his causal theory of war in a speech that an Athenian gives to the Spartan assembly~\parencite{thucydides-jowett} (Book 1, paragraph 76):

\begin{quotation}
    So that, though overcome by three of the greatest things, honor, fear, and profit, we have both accepted the dominion delivered us and refuse again to surrender it, we have therein done nothing to be wondered at nor beside the manner of men.
    Nor have we been the first in this kind, but it hath been ever a thing fixed for the weaker to be kept under by the stronger.
\end{quotation}

Thucydides claimed that three main factors determine how nation-states interact: interest, fear, and  honor.
The \emph{interest} or profit factor just restates a nation's self-interested actions.
Nation-states act against other states to maintain their high-priority national interests within the geographic scope of their power.
These interests include national security, economic security, and sovereignty.
\emph{Fear} refers to the emotionally charged frames through which a nation views world events.
\emph{Honor} refers to the nation's senses of self and entitlement.
Examples include the nineteenth-century US's \emph{manifest destiny} or China's older concept of \emph{Tianxia} or ``all under heaven.''

Allison expands on these factors in his Thucydides-trap model where again the rise of a new power risks war with a dominant power.
He argues that fear is the main cause of war between such a dominant power and a new rising power.
He looked at $16$ such historical power struggles that extend back to the $15^{th}$ century.
He found that $12$ of these power struggles ended in war.
Allison also contends that similar structural dynamics apply elsewhere in international relations.

We parsed Allison's analysis~\parencite{allison2017} to create a FCM of the Thucydides trap for current US-China relations.
The FCM follows Thucydides and uses his three main factors of interest, fear, and honor.
Auxiliary factors give context to the main factors.
The resulting Thucydides-trap FCM has $17$ factors.
Table \ref{tab:trap-desc} lists and describes these factors.

\begin{table}
\centering
	\begin{tabular}{|l|l|}
	\hline
	\textbf{Label} & \textbf{Full Description} \\
	\hline
	\text{FEAR} & \text{Fear}\\
	\text{usd} & \text{US Military/Defense Posture}\\
	\text{chnd} & \text{China Military/Defense Posture}\\
	\text{geod} & \text{Geographical Distance}\\
	\text{ENT} & \text{Sense of Entitlement/Honor}\\
	\text{uspub} & \text{US Public Resentment}\\
	\text{chnpub} & \text{Chinese Public Resentment}\\
	\text{dipl} & \text{Diplomacy Channels \& International Rules}\\
	\text{NUKE} & \text{Nuclear Power/MAD}\\
	\text{ShrdCult} & \text{Shared Culture}\\
	\text{INT} & \text{National Interests Clash}\\
	\text{usecon} & \text{US Economic Dominance}\\
	\text{chnecon} & \text{China Economic Dominance}\\
	\text{econdep} & \text{Economic Interdependence}\\
	\text{ally} & \text{Alliance Network Structural Friction}\\
	\text{shi} & \text{`Shi' or Contextual/Historical Military Momentum}\\
	\text{WAR*} & \text{War, Military Conflict between USA and China}\\
	\hline
	\end{tabular}
\caption{Factors in the Thucydides trap for relations between the US and China in 2017.}
\label{tab:trap-desc}
\end{table}

The Thucydides-trap FCM also uses some of the auxiliary concepts that Allison discussed.
One example is how nuclear weapons affect the chance of all-out war.
Diplomatic institutions and economic dependencies also affect the chance of war.
Treaty and alliance obligations can rapidly induce or expand war as happened in both World Wars.
The FCM in Figure~\ref{fig:ttrap} shows the directed causal edges among the concepts.
Figure \ref{fig:trap-mplot} shows the Thucydides-trap FCM's causal edge matrix $\mathbf{E}$.

We surmised the causal edge strengths based on Allison's discussions.
The final tables below show the textual justifications for the translation into these edge values $e_{ij}$.
Below we present the results of thresholding the magnitudes to their binary extremes.

\begin{figure}[!ht]
\centering
\includegraphics[width=0.75\textwidth]{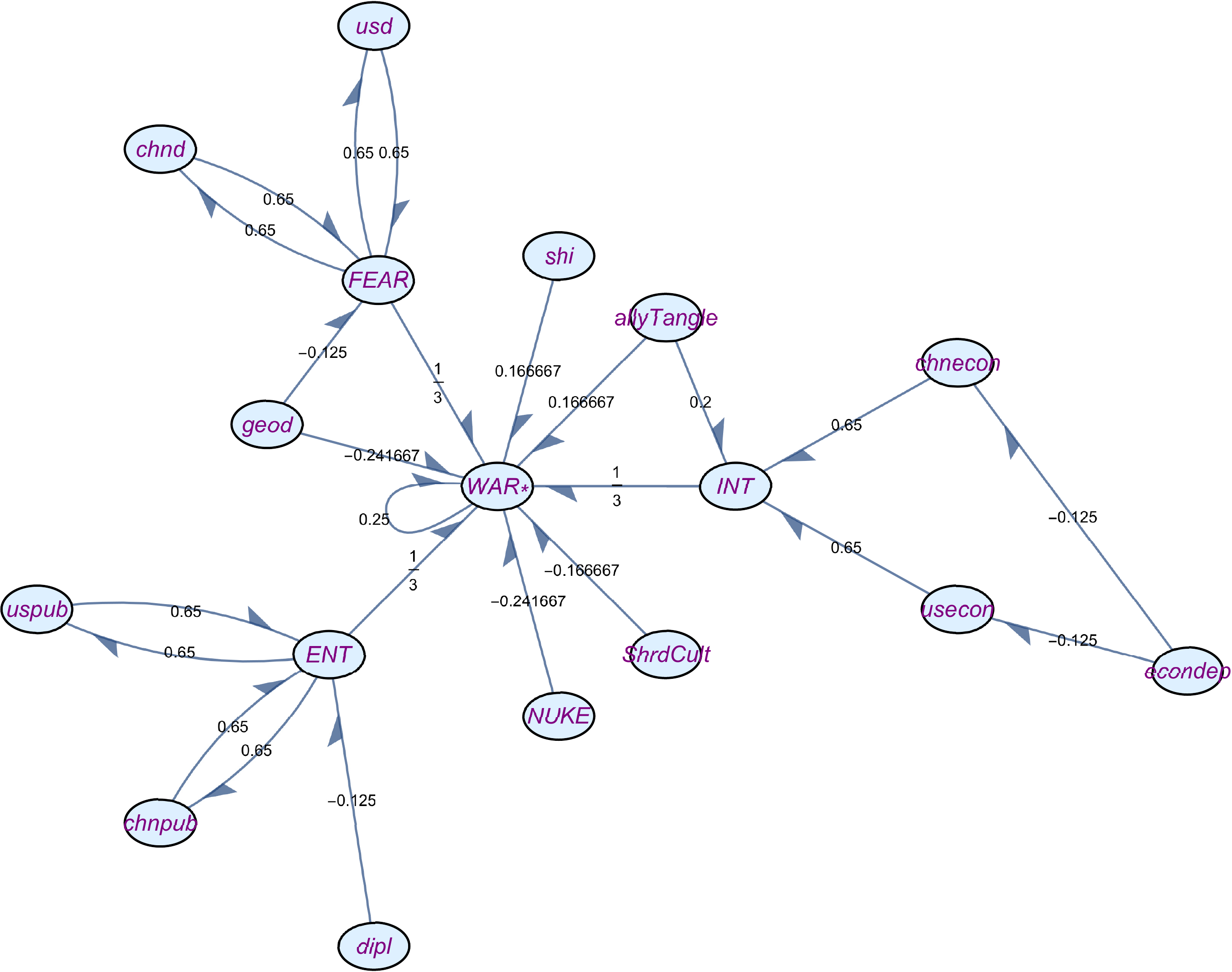}
\caption{
FCM of Allison's \emph{Thucydides' Trap} causal dynamics from \parencite{allison2017}.
}
\label{fig:ttrap}
\end{figure}

\begin{figure}[!tbp]
\centering
\includegraphics[width=0.6\textwidth, keepaspectratio]{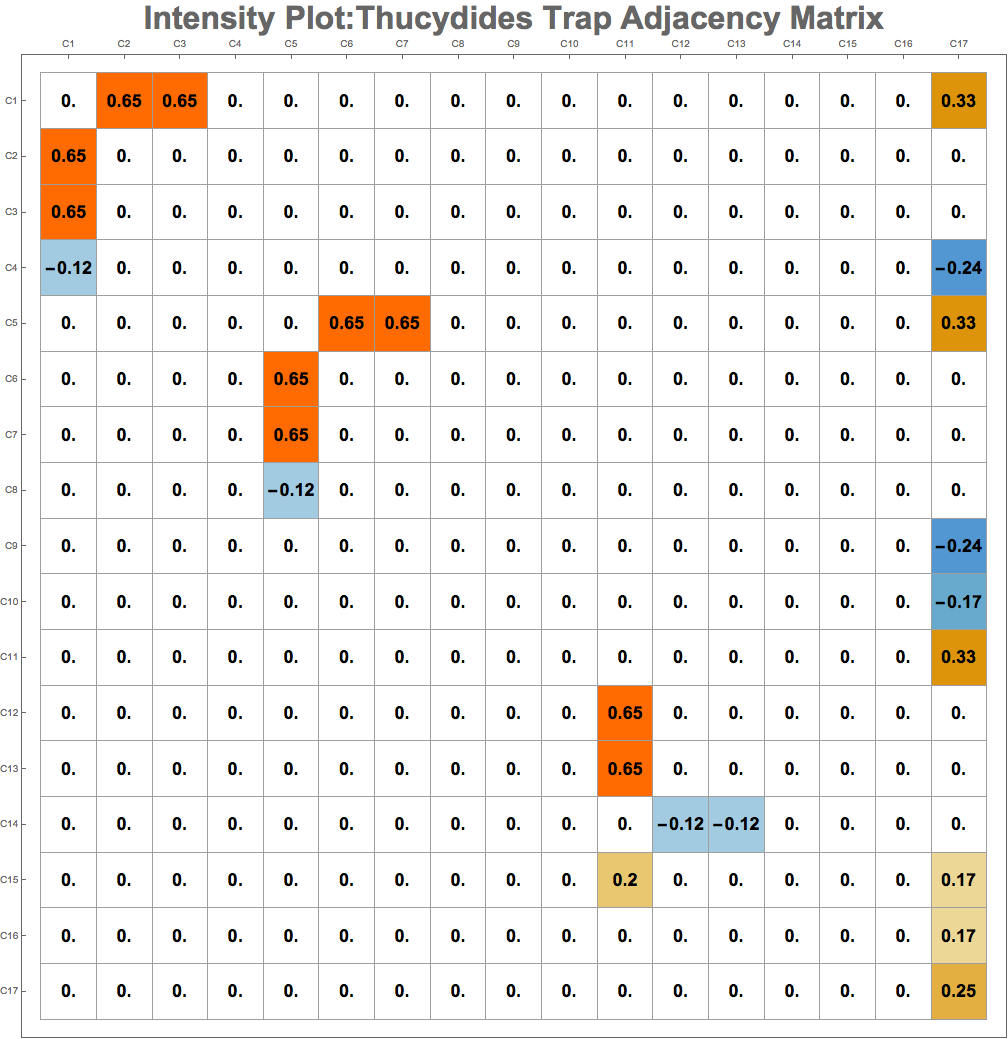}
\caption{Causal edge matrix $\mathbf{E}$ for the Thucydides' trap FCM. \\
$\mathbf{E}$ is the adjacency matrix for the FCM's fuzzy signed directed graph.
Each square shows the fuzzy causal edge value $e_{ij}$.
The value $e_{ij}$ how much the $i^{th}$ concept $C_i$ causes or influences the $j^{th}$ concept $C_j$.
The matrix entries $e_{ij}$ in these FCMs are fuzzy values in the bipolar interval $[-1,1]$.  
Uncolored squares indicate the absence of causal influence.
These matrix intensity plots are larger-scale analogues of the matrix in Equation (\ref{eq:fcm-mtx-ex}) but for a larger set of concepts.
}
\label{fig:trap-mplot}
\end{figure}

The Thucydides' Trap FCM predicted war-type patterns between the US and China more often than it predicted peace-type patterns.
An exhaustive search of the space of possible (clamped) scenarios found that only under $\sim 20\%$ of scenarios led to lasting peace between the dominant power (US) and the rising power (China).
We point out again that these are not representative probabilities because we did not know or estimate the relative probabilities of the input states.
We simply assumed that they were all equally likely.
The FCM rapidly converged to an equilibrium state where ${WAR*}$ was active when the input consisted of US-specific nodes that were stagnant and China-specific nodes that were rising.
The key factors present in peaceful accommodations were significant geographic distance, mutual assured destruction (via nuclear weapons posture), a shared culture, economic interdependence, and the presence of diplomatic channels.

We simulated the FCM dynamical evolution from \emph{trap-like} initial conditions (see Figure~\ref{fig:trap-evol} for initial states and evolution trace).
The test scenario consisted of six causal relations:
\begin{itemize}
	\item The US maintains a strong military or defensive posture.
	\item China is economically rising or already dominant.
	\item US public has high resentment towards China is high.
	\item Both sides are economically interdependent.
	\item Both sides have enough nuclear capability to pose credible threats to each other (sufficient for deterrence).
	\item Strong diplomatic channels exist between both sides.
\end{itemize}
We coded this initial scenario for the FCM concept nodes.
Forward inference gave the sequence of states in Figure \ref{fig:trap-evol}.

\begin{figure}
\centering
\includegraphics[width=0.65\textwidth,height=0.6\textheight,keepaspectratio]{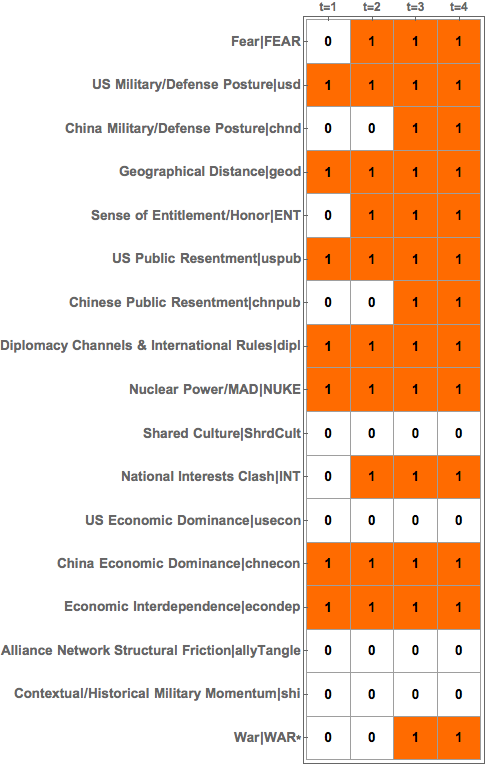}
\caption{Spreading causal activation time slices in the Thucydides-trap FCM.
Each column is a discrete step in time.
The FCM converged in $4$ iterations.}
\label{fig:trap-evol}
\end{figure}

The FCM converged in $4$ iterations to a fixed-point equilibrium state with the ${WAR*}$ node active.
The FCM's state evolution showed that China's economic dominance led to a clash in national interests while the US's defensive posture led to fear.
This led China to ramp up its defensive posture.
The US public's resentment towards China led to a sense of entitlement.
That led in turn to Chinese public resentment.
The clash in national interests, fear, and sense of entitlement or national honor combined to activate the ${WAR*}$ node.
This FCM behavior is similar to the Thucydides-trap dynamics that Allison described.

The FCM's war prediction was robust against many perturbations of the input state.
It persisted despite changes in the activation of nodes like diplomacy and geographic distance.
But activating the Shared-Culture concept node did prevent war.
The FCM also fell out of the war equilibrium when we shut off either the concept node for US Defense Posture or for Chinese Economic Dominance.
These peaceful equilibrium outcomes also appear consistent with Allison's analysis.
Figure \ref{fig:peace-conditions} (left panel) shows the average concept-node activations for initial scenarios that led to peaceful outcomes.

Other analysts may well surmise different fuzzy causal edge values $e_{ij} \in [-1, 1]$ given the same cited text in the tables.
We would expect more agreement on the signs of these edges.
So we tested whether a thresholded version of our properly fuzzy Thucydides-trap FCM made similar equilibrium predictions.
We formed this trivalent Thucydides-trap FCM by replacing all positive edges $e_{ij} > 0$ with $1$ and all negative edges $e_{ij} < 0$ with $-1$.
Zero-valued edges stayed the same.
The trivalent FCM still predicted war-like patterns for most clamped input states.
It predicted peace for only $\sim 15\%$ of all input scenarios compared with $\sim 20\%$ for the original fuzzy Thucydides-trap FCM.
We point out again that we treated all input states as equally likely.
Real-world conflict scenarios are not equally likely.
The right panel of Figure \ref{fig:peace-conditions} shows the similar average concept-node activations for the input scenarios that resolved peacefully in the trivalent FCM.
This counts as evidence that the properly fuzzy Thucydides-trap FCM was reasonably robust to perturbations in the causal edge-value magnitudes.

\begin{figure}[!ht]
\begin{minipage}[r]{\textwidth}
\begin{tabular}{cc}
\includegraphics[width=0.475\columnwidth,clip]{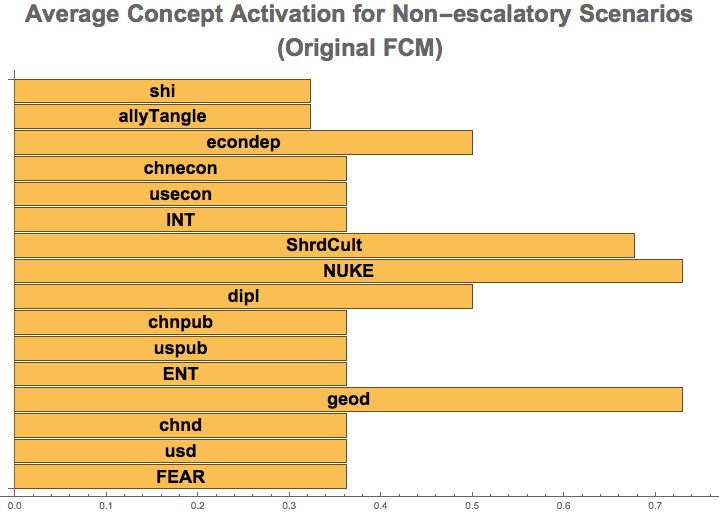}
&
\includegraphics[width=0.475\columnwidth,clip]{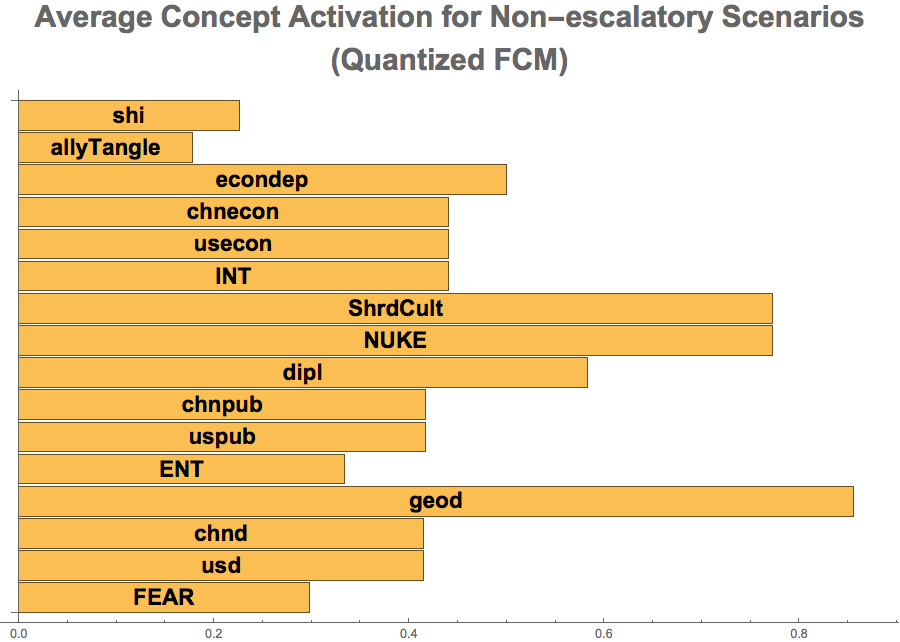}
\end{tabular}
\end{minipage}
\caption{Average node activations for input scenarios that converge to peaceful outcomes for both the original FCM and the thresholded FCM. \\
\textbf{Figure Note}: \textit{Left}: Average concept-node activations for initial scenarios that did not escalate into war.
The bar chart shows that non-escalatory scenarios tended to involve large geographic separation (geod), the availability of nuclear weapons (NUKE), the presence of diplomatic channels (dipl) for resolving issues, and a shared culture between both parties (shrdCult). Economic interdependence (econdep) was also a common feature of peace.\\
\textit{Right}: Trivalent Thucydides-trap FCM:  Average concept-node activations for initial scenarios that did not escalate into war. The trivalent (thresholded) FCM's average activation patterns for peaceful scenarios were similar to those of the original l properly fuzzy FCM.
The key concept nodes associated with peace (geographic distance, nuclear posture, shared culture, presence of diplomatic channels and economic dependence) were the same for both FCMs.
$14.7\%$ of all input scenarios resolved peacefully for the quantized FCM compared with $19.3\%$ for the original FCM.
 This suggests that the Thucydides-trap FCM’s behavior was reasonably robust to mis-specified causal edge-value magnitudes $e_{ij}$. }
\label{fig:peace-conditions}
\end{figure}

Analysts may also disagree on structure and not just numerical edge values.
Other researchers may surmise different causal links (edges) or different relevant concept nodes.
This may depend on an alternative reading of the source text or different domain expertise.
Such disagreements reflect the value of FCM modeling because different FCM models can capture these expert differences and then combine them if needed into a summary FCM.

The US-China FCM shows just one way that FCMs can represent this complex pattern of international dynamics.
Domain experts can instantiate concurring or dissenting causal maps.
Experts and critics alike can then compare or contrast the different theories in this social-scientific domain.
Such comparisons extend beyond just comparisons of the basic representation of different social-scientific theories.
Analysts can also compare the long-term implications of the different theories with more quantitative rigor.
They can compare the equilibrium ``hidden patterns'' in the FCMs' temporal dynamics).

\section{Conclusion} \label{sec:end}

Fuzzy cognitive maps offer a flexible way to model large-scale feedback causal system and to make forward inferences.
Their cyclic structure produces a nonlinear dynamical system that tends to quickly equilibrate to limit-cycle predictions given a causal input or stimulus.
Users can also step through FCM transient or equilibrium states and thereby unfold the dynamical system in time.
The underlying matrix structure of a FCM's directed causal edges permits natural knowledge fusion through simply adding or mixing the augmented FCM causal-edge matrices for any number of experts.
Such knowledge combination tends to improve with the number of combined experts.

Directed acyclic graphs lack these features precisely because they are acyclic.
Their use of probability to describe causal uncertainty is secondary to their lack of cycles in terms of modeling feedback.
Cycles describe feedback in directed graphs.
And combining directed graphs will in general produce several such cycles.
The vector-matrix operations of FCMs also involve much less computation than the probabilistic computations in Bayesian belief networks.
But FCMs cannot produce the precise probability descriptions that DAGs can if the user knows the DAG's corresponding complete joint probability density function and uses the sum-product algorithm.
Imposing this or a related probability structure on a FCM is an area for future research.

Current FCM inference and learning have two key limitations that future research also needs to address.
The first is that FCMs do not easily permit backward chaining.
So they do not in general answer which input caused an observed output effect.
Users cannot simply run the FCM in reverse because of the node nonlinearities.
We instead must exhaustively test all or nearly all input states to see which inputs map to which output equilibria.
This computes the inverse image of each output attractor basin.
It carves up the FCM state-space into attractor regions.
We know for a given output only that the input came from an attractor region.
Future research should address this limitation with new inferencing or other techniques.

The second limitation is more challenging:  How do we infer missing FCM concept nodes?
This just asks how we come up with a new causal hypothesis.
A new node leads to new causal conjectures for all nonzero edges that connect to the new node.
Current adaptive techniques infer and tune the causal edge values only for known concept nodes.
An open research problem is to find data-based techniques that infer new or missing concept nodes in large-scale FCM causal models.
Solutions may include Bayesian priors or rules over node sets or other statistical techniques for model building.

\printbibliography

\section*{Mathematical Appendix:  FCM Causal Influence Theorems} \label{aopendix}

This appendix states and proves the two main theorems on downstream causal influence in a FCM.
Both theorems apply to any two concept nodes in a FCM.
Theorem 1 shows the \emph{transitive} effect that upstream concept node $C_{j_1}$ causes on downstream concept node $C_{j_{k+1}}$ along the lone directed causal path with pairwise directed causal edge values $e_{j_1 j_2}, e_{j_2 j_3}, \ldots , e_{j_k j_{k+1}}$.
Theorem 2 extends this result to the total downstream effect along all acyclic paths from $C_{j_1}$ to $C_{j_{k+1}}$.

We first show next how \emph{causal influence} propagates through a FCM digraph with continuous or smooth concept nodes.
The results still apply to binary or threshold concept nodes if the nodes use a steep logistic or other differentiable sigmoid function to define a soft threshold.
This analysis results in two theorems.
The second theorem extends the first.
The Appendix gives the mathematical details and proofs of the theorems.

The first result describes downstream causal influence for just one causal path from concept node $C_{j_1}$ to $C_{j_{k+1}}$.
A FCM may contain many other directed causal paths from $C_{j_1}$ to $C_{j_{k+1}}$.
So the \emph{total} causal change $\frac{dC_{j_{k+1}}}{dC_{j_1}}$ invokes the more general chain rule that sums over all the partial derivatives in (\ref{eq:partial-path}) in all the paths involved.
The second theorem in the Appendix states this total-causal result.
Discrete versions of the theorems also hold.
They require that one keep track of the discrete time steps as the causal activation flows from one node in a directed path to the next node.

FCM causal influence is a form of spreading activation in a semantic network.
The spreading activation resembles the activation dynamics of asynchronous \emph{feedback} neural networks.
These networks differ in kind from the popular feedfoward neural networks often found ``deep'' neural classifiers.
The spreading activation also roughly resembles the routed messages in message-passing probabilistic inference on DAGs.
Such messages propagate ``evidence'' from observed evidence nodes to a set of output nodes of interest.
Messages encode and quantify how much the state of a node affects beliefs about the state of another node.
Forward-inference on DAGs depends crucially on proper routing of these messages on the causal digraph.
Interlocking nonlinear differential or difference equations describe FCM causal influence.
The analogy with message passing is more accurate when we can step through or otherwise unfold these FCM dynamical systems in time.

The FCM causal influence of node $C_r$ on $C_s$ describes how much the state of $C_r$ affects the state of $C_s$.
The theorems that follow formalize our intuitions on causal influence and generalize previous theorems on the propagation of causal influence~\parencite{osoba-kosko2017}.
They show that causal influence is transitive and easy to track on loopless paths inside a FCM.
The influence also varies directly with the number of causal paths that connect two nodes.
The presence of cycles and loops can induce causal influences that are non-local in time.
The theorems apply directly only to acyclic paths even though users can again step through the causal links in discretized time in an arbitrary path.
The fuzzy or non-binary nature of most edge values can lead to rapid influence die-out because of the product nature of the causal influence.
Feedback loops can reverse such influence die-out even in small-scale FCMs.
The equilibrium result can be a periodic or aperiodic attractor.
It can be a ``hidden pattern'' or forward prediction in the FCM causal tangle.

We show now how FCM nodes influence one another through a weighted product of intervening causal edge strengths $e_{ij}$.
These results describe forward causal chaining along a directed path or summed over all such directed paths that connect two concept nodes.
The results all involve the \emph{transitive} causal product $e_{j_1 j_2} e_{j_2 j_3} \cdots e_{j_k j_{k+1}}$.

Consider first the directed causal path from concept node $C_i$ to node $C_k$ by way of the intervening node $C_j$:
\begin{align}
	C_i \xrightarrow[e_{ij}]{ }   C_j  \xrightarrow[e_{jk}]{} C_k  \;.
\end{align}
Then how does a change in the input node $C_i$ causally affect the downstream node $C_k$?
The chain rule of differential calculus gives a transitive-based product answer for the logistic concept node activation in (\ref{eq:logistic-node}):
\begin{align}
	\frac{\partial C_k }{\partial C_i}   =&  \frac{\partial C_k }{\partial C_j}\frac{\partial C_j }{\partial C_i} \\
		=&  \frac{\partial C_k }{\partial x_k} \frac{\partial x_k }{\partial C_j}\frac{\partial C_j }{\partial x_j} \frac{\partial x_j }{\partial C_i} \\
		=&  C_k(x_k) (1 - C_k(x_k)) e_{jk} C_j(x_j) (1 - C_j(x_j)) e_{ij}  \\
		=&  e_{ij} e_{jk} \psi_{j,k}   \label{eq:two-node-chain}
\end{align}
using (\ref{eq:logistic-factor}) - (\ref{eq:logistic-derivative}) if we define $\psi_{j,k} = C_j C_k (1 - Cj)(1 - C_k)$.
The weighting function $\psi_{j,k} \ge 0$ is maximal when $C_j = 1 - C_j = \frac{1}{2} = C_k = 1 - C_k$ holds for the fuzzy concept nodes $C_j$ and $C_k$.

So the induced causal effect of a change in $C_i$ depends directly on the transitive causal-edge product $e_{ij} e_{jk}$.
This causal influence decays in intensity the lesser $C_j$ or $C_k$ fires or occurs.
The edge product $e_{ij} e_{jk}$ is negative if exactly one of the edge values is negative.
It is positive otherwise.

The causal-influence result (\ref{eq:two-node-chain}) extends directly to longer causal chains.
Suppose there is a directed causal path of length $k$ from the initial concept node $C_{j_1}$ to the final node $C_{j_{k+1}}$:
\begin{align}
	C_{j_1} \xrightarrow[e_{j_1  j_2}]{ }   C_{j_2}  \xrightarrow[e_{j_2 j_3}]{}   \cdots \xrightarrow[e_{j_{k} j_{k+1}}]{} C_{j_{k+1}}  \;. \label{eq:k-length-chain}
\end{align}
Then the chain rule and (\ref{eq:logistic-factor}) - (\ref{eq:logistic-derivative}) again give the influence of $C_{j_1}$ on $C_{j_k}$ as a weighted product of the intervening causal edge strengths:
\begin{align}
	\frac{\partial C_{j_{k+1}}}{\partial C_{j_1}}  =  \prod_{l=1}^{k} e_{j_l j_{l+1}}  \psi_{j_2, j_3, \ldots , j_{k+1}}  \label{eq:k-node-chain}
\end{align}
where now the nonnegative weighting function $\psi_{j_2, j_3, \ldots , j_{k+1}}$ is the double product $\psi_{j_2, j_3, \ldots , j_{k+1}} = \prod_{l=2}^{k+1} C_{j_l} \prod_{l=2}^{k+1} (1 - C_{j_l})$.
The edge product $\prod_{l=1}^{k} e_{j_l j_{l+1}}$ is positive if the number of negative edges is even.
It is negative if the number of negative edges is odd.
The magnitude of the change $\frac{\partial C_{j_{k+1}}}{\partial C_{j_1}}$ can only decrease as the causal chain lengthens.
The fuzziness or partial firing of the concept nodes only exacerbates this monotone causal decay.

The causal influence in (\ref{eq:k-node-chain}) still holds if we replace the logistic activation function (\ref{eq:node-nonlinearity}) of concept node $C_j$ with an arbitrary monotonically nondecreasing functions $\Phi_j$.
Then $\frac{\partial C_{j_l}}{\partial x_{j_l}} \ge 0$ and so $\psi_{j_2, j_3, \ldots , j_{k+1}} \ge 0$ because the weighting function is just the product of these activation partial derivatives.
We can now state and prove Theorems 1 and 2 on causal influence in FCMs.
Theorem 1 summarizes and extends the above argument for an arbitrary single directed path in a FCM.
Theorem 2 further extends the argument to summing over all acylic paths between two such concept nodes.

\begin{theorem}  \bf Partial Causal Influence in Fuzzy Cognitive Maps.
Suppose that a fuzzy cognitive map has $n$ concept nodes $C_j$ and $n^2$ directed causal edges $e_{ij}$.
Suppose further that the concept nodes have monotonically nondecreasing activations:  $\frac{\partial C_j }{\partial x_j} \ge 0$ where the argument $x_j$ of $C_j(x_j)$ has the same inner-product form as in (\ref{eq:node-nonlinearity}).
Then the causal influence of the concept node $C_{j_1}$ on the downstream node $C_{j_{k+1}}$ of the length-$k$ directed causal chain
\begin{align}
	C_{j_1} \xrightarrow[e_{j_1  j_2}]{ }   C_{j_2}  \xrightarrow[e_{j_2 j_3}]{}   \cdots \xrightarrow[e_{j_{k} j_{k+1}}]{} C_{j_{k+1}}
\end{align}
is a nonnegatively weighted product of the intervening causal edge strengths $e_{j_1 j_2}, \ldots, e_{j_k j_{k+1}}$:
\begin{align}
	\frac{\partial C_{j_{k+1}}}{\partial C_{j_1}}  =  \prod_{l=1}^{k} e_{j_l j_{l+1}}  \psi_{j_2, j_3, \ldots , j_{k+1}}
\end{align}
where the weighting function $\psi_{j_2, j_3, \ldots , j_{k+1}}$ has the form
\begin{align}
	 \psi_{j_2, j_3, \ldots , j_{k+1}} = \prod_{l=2}^{k+1} \frac{\partial C_{j_l} }{\partial x_{j_l}}  \;.
\end{align}
\label{thm:partial-cause}
\end{theorem}

\begin{proof}
The result follows from iterated applications of the chain rule:
\begin{align}
	\frac{\partial C_{j_{k+1}}}{\partial C_{j_1}}   =&  \frac{\partial C_{j_{k+1}}}{\partial C_{j_k}} \frac{\partial C_{j_{k}}}{\partial C_{j_{k-1}}} \cdots \frac{\partial C_{j_2} }{\partial C_{j_1}}  \label{eq:partial-path}   \\
		=& \frac{\partial C_{j_{k+1}}}{\partial x_{j_{k+1}}} \frac{\partial x_{j_{k+1}}}{\partial C_{j_{k}}} \frac{\partial C_{j_{k}}}{\partial x_{j_{k}}} \frac{\partial x_{j_{k}}}{\partial C_{j_{k-1}}}   \cdots \frac{\partial C_{j_2} }{\partial x_{j_2}} \frac{\partial x_{j_2}}{\partial C_{j_1}} \\
		=& \frac{\partial C_{j_{k+1}}}{\partial x_{j_{k+1}}} e_{j_{k} j_{k+1}} \frac{\partial C_{j_{k}}}{\partial x_{j_{k}}}   e_{j_{k-1} j_{k}} \cdots \frac{\partial C_{j_2} }{\partial x_{j_2}} e_{j_{k} j_{k+1}}   \\
		=&  \prod_{l=1}^{k} e_{j_l j_{l+1}} \prod_{l=2}^{k+1} \frac{\partial C_{j_l} }{\partial x_{j_l}} \\
		=&  \prod_{l=1}^{k} e_{j_l j_{l+1}}\psi_{j_2, j_3, \ldots , j_{k+1}}  \label{eq:k-node-chain-monotone} \;.
\end{align}
\end{proof}

Theorem 1 implies that the sign of the edge product $\prod_{l=1}^{k} e_{j_l j_{l+1}}$ depends on only the number of negative edges.
The edge product $\prod_{l=1}^{k} e_{j_l j_{l+1}}$ is positive if the number of negative edges is even.
It is negative if the number of negative edges is odd.
The magnitude of the change $\frac{\partial C_{j_{k+1}}}{\partial C_{j_1}}$ can only decrease as the causal chain lengthens.
The fuzziness or partial firing of the concept nodes only exacerbates this monotone causal decay.
The result is often fairly rapid die-out of the causal influence in practice.

A FCM often has more than one directed causal path from $C_{j_1}$ to $C_{j_{k+1}}$ because a FCM is not a tree in general.
Then the total causal change $\frac{dC_{j_{k+1}}}{dC_{j_1}}$ sums the single-path or partial causal influences (\ref{eq:partial-path}) over all paths from $C_{j_1}$ to $C_{j_{k+1}}$.
Let $\lambda(C_r, C_s)$ denote a directed fuzzy causal path from $C_r$ to $C_s$:
\begin{align}
	\lambda(C_r, C_s) &= C_r \xrightarrow[e_{r \cdot}]{ } \cdots \xrightarrow[e_{\cdot s}]{} C_s \;.
\end{align}
Let $\vec{\lambda}$ denote the set of node subscripts in the directed path:
\begin{align}
	\vec{\lambda} &= \{ r, \cdots , s \} \;.
\end{align}
We can restate the  causal influence along the lone path $\vec{\lambda}$ in Theorem \ref{thm:partial-cause} as
\begin{align}
	\partial \left(C_r, C_s; \vec{\lambda} \right)	= \psi_{\mathbf{\lambda}_2, \mathbf{\lambda}_3, \ldots , s} \prod_{i=1}^{|\vec{\lambda}|-1} e_{\mathbf{\lambda}_i \mathbf{\lambda}_{i+1}}
\end{align}
where $|\vec{\lambda}|$ denotes the number of nodes in the path.
Then we can likewise state the \emph{total} causal influence $D\left(C_r, C_s \right)$  of $C_r$ on $C_s$ as the sum of all single-path causal influences $\partial \left(C_r, C_s; \vec{\lambda}_m \right)$ over all such paths $\Lambda = \{\vec{\lambda}_m \}_m$ from $C_r$ to $C_s$:
\begin{align}
	D\left(C_r, C_s \right) = \sum_{m=1}^{|\Lambda|} \partial \left(C_r, C_s; \vec{\lambda}_m \right) \;.
\end{align}
Taking this path sum gives Theorem 2.

\begin{theorem}[\bf Total Causal Influence Over Acyclic FCM Paths]
Suppose an $n$-node FCM satisfies the hypothesis of Theorem 1.
Suppose that there are no cycles in any of the directed paths from concept node $C_r$ to $C_s$.
Then the total causal influence of concept node $C_r$ on $C_s$ is
\begin{align}
	D\left(C_r, C_s \right) = \sum_{m=1}^{|\Lambda|} \partial \left(C_r, C_s; \vec{\lambda}_m \right)
\end{align}
where $\Lambda = \{\vec{\lambda}_m \}$ is the complete set of directed paths $\lambda_m$ from $C_r$ to $C_s$
\label{thm:total-cause-dag}
\end{theorem}

\section*{Appendix: Textual Justifications for Thucydides Trap FCM}

\begin{table}
\centering
	\begin{tabular}{|l|c|p{3.5in}|}
	\hline
	\textbf{Edge} & \textbf{Sign} & \textbf{Text quote} \\
	\hline
  \text{FEAR/INT/ENT} $\rightarrow$ \text{WAR*} & \text{+ve} &
	\begin{minipage}[t]{3.4in}
		Pg. 39: ``Thucydides identifies three primary drivers fueling this dynamic that lead to war: interests, fear, and honor.''\\
	  ``[Interests refer to the] survival of the state and its sovereignty in making decisions in its domain free from coercion...''\\
	  ``[...] ruling powers' fears often fuel misperceptions and exaggerate dangers.''\\
	  ``[...] Thucydides's concept of honor encompasses what we now think of as a state's sense of itself, its convictions about the recognition and respect it is due, and its pride. As Athens' power grew over the fifth century, so too did its sense of entitlement.''
	\end{minipage}
	\\
	\hline
  \text{ShrdCult} $\rightarrow$ \text{WAR*} & \text{-ve} &
	\begin{minipage}[t]{3.4in}
		Pg. 200: ``Cultural commonalities can help prevent conflict''\\
		Pg. 136: ``... the fundamental source of conflict in the post-Cold War world would [be] cultural.'' (quoting Huntington)
	\end{minipage}
	\\
	\hline
  \text{ally} $\rightarrow$ \text{WAR*} & \text{+ve} &
	\begin{minipage}[t]{3.4in}
		Pg. 211--212: ``Alliances can be a fatal attraction. [...] coalitions have sought to create a balance of power to maintain regional peace and security. But such alliances also create risks -- since alliance ties run in both directions.''\\
		Pg. 57--58: ``Britain had no vital national interests at stake in the Balkans. Nevertheless, it was pulled into the fire, partly because of entangling alignments...''
	\end{minipage}
	\\
	\hline
  \text{NUKE} $\rightarrow$ \text{WAR*} & \text{-ve} &
	\begin{minipage}[t]{3.4in}
		Pg. 206--210: ``nuclear weapons have no precedent.'' \\
		``Under [conditions of mutually assured destruction], one state's decision to kill another is simultaneously a choice to commit national suicide.'' \\
		``China has also developed a nuclear arsenal so robust that it creates a [...] version of MAD with the United States.''
	\end{minipage}
	\\
	\hline
	\end{tabular}
\caption{Table of textual justifications for the causal edges in the FCM of Allison's USA-China Thucydides Trap. Primary reference: \parencite{allison2017}. Secondary sources: \parencite{allison2015, thucydides-jowett, huntington1993, huntington1997}.}
\label{tab:trap-edge-just-1}
\end{table}

\begin{table}
\centering
	\begin{tabular}{|l|c|p{3.5in}|}
	\hline
	\textbf{Edge} & \textbf{Sign} & \textbf{Text quote} \\
	\hline
  \text{shi} $\rightarrow$ \text{WAR*} & \text{+ve} &
	\begin{minipage}[t]{3.4in}
		Pg. 148: ``... the evolving context in which a strategic situation occurs is critical, because it determines the \emph{shi} of that situation. [This] most closely describes the momentum inherent in any given situation...'' \\
		Pg. 175: ``America's history of [recent] military interventions would play an outsized role in shaping Washington's response.'' \\
		Pg. 215: ``Bismarck [...] described statecraft as essentially listening for the footsteps of God and then grabbing the hem of His garment as He goes by.''
	\end{minipage}
  \\
	\hline
  \text{econdep} $\rightarrow$ \text{usecon/chnecon} & \text{-ve} &
	\begin{minipage}[t]{3.4in}
		Pg. 194: ``States can be embedded in larger economic [...] institutions that constrain historically normal [conflict behaviors].'' \\
		Pg. 210: ``Thick economic interdependence raises the cost and thus lowers the likelihood of war.''
	\end{minipage}
	\\
	\hline
  \text{geod} $\rightarrow$ \text{FEAR/WAR*} & \text{-ve} &
	\begin{minipage}[t]{3.4in}
		Pg. 109: `` 'Making China Great Again' means [...] reestablishing control over the territories of 'Greater China'.''\\
		Pg. 161``...there is currently little chance of an accidental collision between US and Chinese ships. [...] the 'tyranny of distance' raises questions about America's ability to sustain a campaign against China in [the East and South China Seas].''
	\end{minipage}
	\\
	\hline
  \text{dipl} $\rightarrow$ \text{ENT/INT} & \text{-ve} &
	\begin{minipage}[t]{3.4in}
		Pg 190: ``Higher authorities can help resolve rivalry without war. [...] to the extent that states can be persuaded to defer to the constraints and decisions of supranational authorities or legal frameworks, [...] these factors can play significant roles in managing conflicts that would otherwise end in war.''\\
		Pg. 198:``Wily statemen make a virtue of necessity and distinguish between needs and wants.''
	\end{minipage}
	\\
	\hline
	\end{tabular}
\caption{Table of textual justifications for the causal edges in the FCM of Allison's USA-China Thucydides Trap. Primary reference: \parencite{allison2017}. Secondary sources: \parencite{allison2015, thucydides-jowett, huntington1993, huntington1997}.}
\label{tab:trap-edge-just-2}
\end{table}

\begin{table}
\centering
	\begin{tabular}{|l|c|p{3.5in}|}
	\hline
	\textbf{Edge} & \textbf{Sign} & \textbf{Text quote} \\
	\hline
  \text{WAR*} $\rightarrow$ \text{WAR*} & \text{+ve} &
	\begin{minipage}[t]{3.4in}
		Continuity constraint: states at war are likely to remain in conflict barring military resolution or exogenous shock.
	\end{minipage}
	\\
	\hline
  \text{chnecon/usecon} $\rightarrow$ \text{INT} & \text{+ve} &
	\begin{minipage}[t]{3.4in}
		Pg. 47: ``The contest between a rising and ruling power often intensifies competition over scarce reseources. When an expanding economy compels the first to reach further afield to secure essential commodities [...] the competition can become a resource scramble.''
	\end{minipage}
	\\
	\hline
  \text{chnpub/uspub} $\rightarrow$ \text{ENT} & \text{+ve} &
	\begin{minipage}[t]{3.4in}
		Pg. 37: ``Backing down was a non-starter [...] the Athenian populace was unwilling to bow to Spartan demands...''\\
		Pg. 120: ``...'winning or losing public support is an issue that concerns the [Chinese Communist Party]'s survival or extinction.' '' (quoting Xi Jinping)\\
		Pg. 171: ``the White House Situation Room cannot back down: video of the ship's wreckage and stranded US sailors on cable news and social media has made that impossible'' (part of a hypothetical US-China escalation scenario)\\
		Pg. 171: ``As millions of it citizens' social media postings are reminding the government, after its century of humiliation at the hands of sovereign powers, the ruling Communist Party has promised: `never again.' '' (in the same hypothetical US-China escalation scenario)
	\end{minipage}
	\\
	\hline
	\end{tabular}
\caption{Table of textual justifications for the causal edges in the FCM of Allison's USA-China Thucydides Trap. Primary reference: \parencite{allison2017}. Secondary sources: \parencite{allison2015, thucydides-jowett, huntington1993, huntington1997}.}
\label{tab:trap-edge-just-3}
\end{table}

\begin{table}
\centering
	\begin{tabular}{|l|c|p{3.5in}|}
	\hline
	\textbf{Edge} & \textbf{Sign} & \textbf{Text quote} \\
	\hline
  \text{chnd/usd} $\rightarrow$ \text{FEAR} & \text{+ve} &
	\begin{minipage}[t]{3.4in}
		Pg. 54: ``Because states can never be certain about each other's intent, they focus instead on [military] capabilities. Defensive actions by one power can often seem threatening to its opponent...''
	\end{minipage}
	\\
	\hline
	\text{ShrdCult} $\leftrightarrow$ \text{econdep} & \text{+ve} &
	\begin{minipage}[t]{3.4in}
		\parencite{huntington1993} Pg. 34: ``...cultural difference exacerbates economic conflict. [...] the antipathies are not racial but cultural. The basic values, attitudes, behavioral patterns of the two societies could hardly be more different. The economic issues between the United States and Europe are no less serious than those between the United States and Japan, but they do not have the same political salience and emotional intensity because the differences between American culture and European culture are so much less than those between American civilization and Japanese civilization.''
	\end{minipage}
	\\
	\hline
	\end{tabular}
\caption{Table of textual justifications for the causal edges in the FCM of Allison's USA-China Thucydides Trap. Primary reference: \parencite{allison2017}. Secondary sources: \parencite{allison2015, thucydides-jowett, huntington1993, huntington1997}.}
\label{tab:trap-edge-just-4}
\end{table}

\end{document}